%% file: rkhs_app_bandit.tex
\documentclass{article}
\pdfoutput=1

\usepackage{microtype}
\usepackage{graphicx}
\usepackage{subfigure}
\usepackage{booktabs} 

\usepackage{hyperref}

\usepackage{algorithm, algorithmic}
\usepackage{natbib}

\usepackage{amsthm}
\usepackage{amssymb, amsmath}
\usepackage{enumitem}

\newtheorem{thm}{\textsc{Theorem}}
\newtheorem{prop}[thm]{\textsc{Proposition}}
\newtheorem{lem}[thm]{\textsc{Lemma}}
\newtheorem{cor}[thm]{\textsc{Corollary}}

\theoremstyle{definition}
\newtheorem{rem}[thm]{\textsc{Remark}}
\newtheorem{dfn}[thm]{\textsc{Definition}}

\DeclareMathOperator{\Tr}{Tr}
\DeclareMathOperator{\argmax}{argmax}
\DeclareMathOperator{\prob}{Pr}

\newcommand{\RR}{\mathbb{R}}

\newcommand{\ZZ}{\mathbb{Z}}

\newcommand{\trn}{\mathrm{T}}
\newcommand{\rank}{\mathrm{rank} \hspace{2pt}}

\newcommand{\normdist}{\mathcal{N}}
\newcommand{\gl}{\mathrm{GL}}

\newcommand{\ex}[2][{}]{\mathbf{E}_{#1}\left[#2\right]}
\newcommand{\bl}{\left(}
\newcommand{\br}{\right)}
\newcommand{\aset}{\mathcal{A}}
\newcommand{\filt}{\mathcal{F}}
\newcommand{\hilb}[1]{\mathcal{H}_{K}\left(#1\right)}
\newcommand{\gsk}{\mathrm{SE}}
\newcommand{\rqk}{\mathrm{RQ}}
\newcommand{\mtk}{\mathrm{Mat\acute{e}rn}}
\newcommand{\linf}[1]{L^{\infty}\left(#1\right)}
\newcommand{\omdsc}{\widehat{\Omega}}

\newcommand{\wx}{\widetilde{x}}
\newcommand{\thetahat}{\hat{\theta}}
\newcommand{\otilde}{\widetilde{O}}

\newcommand{\algucb}{\text{APG-UCB}}
\newcommand{\algpe}{\text{APG-PE}}
\newcommand{\algthmp}{\text{APG-TS}}
\newcommand{\algucbh}{APG-UCB}

\newcommand{\algpeh}{APG-PE}
\newcommand{\algexp}{APG-EXP3}

\newcommand{\myciteauthoryear}{\citet}
\newcommand{\parencite}{\cite}
\newcommand{\Input}[1]{\STATE \textbf{Input}: #1}
\newcommand{\Output}[1]{\STATE \textbf{Output}: #1}
\newcommand{\err}{\epsilon}
\newcommand{\admerr}{\mathfrak{e}}
\newcommand{\piexp}{\pi_{\mathrm{exp}}}
\newcommand{\betaigp}{\beta^{\text{IGP-UCB}}}

\title{Approximation Theory Based Methods for RKHS Bandits}
\author{Sho Takemori\thanks{sho.takemori.py@fujifilm.com,takemorisho@gmail.com}, Masahiro Sato}
\date{}
\begin{document}
\maketitle
\begin{abstract}
The RKHS bandit problem (also called kernelized multi-armed bandit problem)
is an online optimization problem of non-linear functions with noisy feedback.
Although the problem has been extensively studied,
there are unsatisfactory results for some problems compared to
the well-studied linear bandit case.
Specifically, there is no general algorithm for the adversarial RKHS bandit problem.
In addition, high computational complexity of existing algorithms hinders practical application.
We address these issues by considering a novel amalgamation
of approximation theory and the misspecified linear bandit problem.
Using an approximation method,
we propose efficient algorithms for the stochastic
RKHS bandit problem and the first general algorithm for the adversarial RKHS bandit problem.
Furthermore,
we empirically show that one of our proposed methods has
comparable cumulative regret to IGP-UCB and its running time is much shorter.
\end{abstract}

\section{Introduction}
The RKHS bandit problem (also called kernelized multi-armed bandit problem)
is an online optimization problem of non-linear functions with noisy feedback.
\myciteauthoryear{srinivas2010gaussian}
studied a multi-armed bandit problem where the reward function
belongs to the reproducing kernel Hilbert space (RKHS) associated with a kernel.
In this paper, we call this problem the (stochastic) RKHS bandit problem.
Although the problem has been studied extensively,
some issues are not completely solved yet.
In this paper, we focus on mainly two issues:
non-existence of general algorithms for the adversarial RKHS bandit problem
and high computational complexity for the stochastic RKHS bandit algorithms.

First, as a non-linear generalization of the classical adversarial linear bandit problem,
\myciteauthoryear{chatterji2019online} proposed the adversarial RKHS bandit problem,
where a learner interacts with a sequence of any functions from the RKHS with bounded norms.
However, they only consider the kernel loss, i.e., a loss function of the form $x \mapsto K(x, x_0)$,
where $x_0$ is a fixed point.
Considering functions in the RKHS can be represented as infinite
linear combinations of such functions,
kernel loss is a very special function in the RKHS.
Therefore,
there are no algorithms for the adversarial RKHS bandits with general
loss (or reward) functions.

Next, we discuss the efficiency of existing methods for the stochastic RKHS bandit problem.
We note that most of the existing methods have regret guarantees at the cost of high computational complexity.
For example, IGP-UCB \parencite{chowdhury2017kernelized} requires
matrix-vector multiplication of size $t$ for each arm at each round $t=1, \dots, T$.
Therefore, the total computational complexity up to round $T$
is given as $O(|\aset|T^3)$, where $\aset$ is the set of arms.
To address the issue,
\citet{calandriello2020near} proposed BBKB
and proved its total computational complexity is given as
$\otilde(|\aset|T\gamma_T^2 + \gamma_T^4)$,
where $\aset \subset \Omega$ is the set of arms,
$\Omega$ is a subset of a Euclidean space $\RR^d$,
and $\gamma_T$ is the maximum information gain \citep{srinivas2010gaussian}.
If the kernel is a squared exponential kernel, then
since $\gamma_T = \otilde(\log^d (T))$%
\footnote{In this paper, we use $\otilde$ notation to ignore $\log^c(T)$ factor, where $c$ is a universal constant.}
\citep{srinivas2010gaussian},
ignoring the polylogarithmic factor,
BBKB's computational complexity is nearly linear in $T$.
However, the coefficient $|\aset|$ in the term is large in general.

In this paper, we address these two issues
by considering a novel amalgamation of approximation theory \citep{wendland2004scattered}
and the misspecified linear bandit problem \cite{lattimore2020learning}.
That is, we approximately reduce the RKHS bandit problem
to the well-studied linear bandit problem.
Here, because of an approximation error, the model is a misspecified linear model.
Ordinary approximation methods (such as Random Fourier Features or Nystr\"om embedding)
basically aim to approximate kernel $K(x, y)$ by an inner product of finite dimensional vectors.
However, to reduce the RKHS bandits to the linear bandits,
we want to approximate a function $f$ in the RKHS $\hilb{\Omega}$ by a function $\phi$
in a finite dimensional subspace
so that $\| f -  \phi \|_{L^{\infty}(\Omega)}$ is small.
Since the usual approximation methods are not appropriate for the purpose,
in this paper, we utilize a method developed in the approximation theory literature
called the $P$-greedy algorithm \citep{de2005near}
to minimize the $L^{\infty}$ error.
More precisely,
we shall introduce that
any function $f$ in the RKHS is approximately equal (in terms of the $L^{\infty}$ norm)
to a linear combination of $D_{q, \alpha}(T)$ functions,
where $q, \alpha > 0$ are parameters and
$D_{q, \alpha}(T)$ is the number of functions
(or equivalently points)
returned by the $P$-greedy algorithm (Algorithm \ref{alg:nbasispgreedy})
with admissible error $\admerr = \frac{\alpha}{T^q}$.
If $K$ is sufficiently smooth, $D_{q, \alpha}(T)$ is much smaller
than $T$ and $|\aset|$.
By this approximation,
we can tackle the original RKHS bandit problem
by applying an algorithm for the misspecified linear bandit problem.


\subsection*{Contributions}
To state contributions, we introduce terminology for kernels.
In this paper, we consider two types of kernels: kernels with infinite smoothness
and those with finite smoothness with smoothness parameter $\nu$
(we provide a precise definition in \S \ref{sec:res-approximation}).
Examples of the former include Rational Quadratic (RQ) and Squared Exponential (SE) kernels
and those of the latter include the Mat\'ern kernels with parameter $\nu$.
The latter type of kernels also include a general class of kernels
 that belong to $C^{2\nu}(\Omega \times \Omega)$
with $\nu \in \frac{1}{2}\ZZ_{> 0}$ and satisfy some additional conditions.
Let $D_{q, \alpha}(T) \in \ZZ_{>0}$ be as before.
Then, in \S \ref{sec:res-approximation}, we shall introduce that
$D_{q, \alpha}(T) = O\left((q\log T - \log(\alpha))^d\right)$ if $K$ has infinite smoothness
and $D_{q, \alpha}(T) = O\left(\alpha^{-d/\nu} T^{dq/\nu}\right)$ if $K$ has finite smoothness.
Our contributions are stated as follows:
\begin{enumerate}
  \item We apply an approximation method
  that has not been applied to the RKHS bandit problem and
  reduce the problem to the well-studied (misspecified) linear bandit problem.
  This novel reduction method has potential to tackle issues other than the ones we deal with in this paper.

  \item We propose \algexp{} for the adversarial RKHS bandit problem,
  where APG stands for an Approximation theory based method using $P$-Greedy.
  We prove its expected cumulative regret is upper bounded by
  $\otilde\left( \sqrt{T D_{1, \alpha}(T)  \log\left(|\aset|\right) } \right)$,
  where $\alpha = \log(|\aset|)$.
  To the best of our knowledge, this is the first method for the adversarial RKHS bandit problem
  with general reward functions.

  \item We propose a method for the stochastic RKHS bandit problem called $\algpe$
  and prove its cumulative regret is given as $
      \otilde\left(
      \sqrt{T D_{1/2, \alpha}(T)  \log\left(\frac{|\aset|}{\delta}\right)
      }
      \right),
  $ with probability at least $1-\delta$ and its total computational complexity is given as
  $\otilde \left( (|\aset| + T) D_{1/2, \alpha}^2(T) \right)$.
  We note that the total computational complexity is generally much better than
  that of the state of the art result
  $\otilde( |\aset|T\gamma_T^2 + \gamma_T^4)$
  \citep{calandriello2020near}.

  \item We propose \algucbh{} as an approximation of IGP-UCB and
  provide an upper bound of its cumulative regret
  if $q \ge 1/2$ and
  prove that its total computational complexity is given as $O(|\aset|T D_{q, \alpha}^2(T))$.

  If we take the parameter $q$ so that $q > 3/2$,
  then we shall show that $R_{\algucb}(T)$
  is upper bounded by $4 \betaigp_T \sqrt{\gamma_T T}
  + O(\sqrt{T\gamma_T} T^{(3/2-q)/2}
 + \gamma_T T^{1-q})$, where we define $\betaigp_T$ in \S \ref{sec:main-results}.
  Since the upper bound for the cumulative regret of IGP-UCB is also given as
  $4 \betaigp_T \sqrt{\gamma_T (T + 2)}$,
  \algucbh{} has asymptotically the same
  regret upper bound as that of IGP-UCB in this case.
  If the kernel has infinite smoothness
  or finite smoothness with sufficiently large $\nu$ (i.e., $\nu > 3d/2$),
  then this method is more efficient than IGP-UCB,
  whose computational complexity is $O(|\aset|T^3)$.

  \item In synthetic environments,
  we empirically show that \algucbh{} has almost the same cumulative regret as that of IGP-UCB
  and its running time is much shorter.
\end{enumerate}





\section{Related Work}
First, we review previous works on the adversarial RKHS bandit problem.
There are almost no existing results concerning the adversarial RKHS bandit problem except
for \citep{chatterji2019online}.
They also used an approximation method to solve the problem,
but their approximation method can handle only a limited case.
Therefore, there are no existing algorithms for the adversarial RKHS bandit problem
with general reward functions.
Next, we review existing results for the stochastic RKHS bandit problem.
\myciteauthoryear{srinivas2010gaussian}
studied a multi-armed bandit problem,
where the reward function is assumed to be sampled from a Gaussian process or
belongs to an RKHS.
\myciteauthoryear{chowdhury2017kernelized}
improved the result of \myciteauthoryear{srinivas2010gaussian} in the RKHS setting
and proposed two methods called IGP-UCB and GP-TS.
\myciteauthoryear{valko2013finite}
considered a stochastic RKHS bandit problem, where the arm set $\aset$ is finite
and fixed,
prpoposed a method called SupKernelUCB, and
 proved a regret upper bound $\otilde(\sqrt{T \gamma_T \log^3(|\aset|T/\delta)})$.
To address the computational inefficiency in the stochastic RKHS bandit problem,
\myciteauthoryear{mutny2018efficient} proposed Thompson Sampling and UCB-type
algorithms using an approximation method called Quadrature Fourier Features which
is an improved method of Random Fourier Features \parencite{rahimi2008random}.
They proved that the total computational complexity of their methods is given as
$\otilde(|\aset|T \gamma_T^2)$.
However, their methods can be applied to only a very special class of kernels.
For example, among three examples introduced in \S \ref{sec:prob-form},
only SE kernels satisfy their assumption unless $d =1$.
Our methods work for general symmetric positive definite kernels
with enough smoothness.
\myciteauthoryear{calandriello2020near}
proposed a method called BBKB and proved its regret is upper bounded by
$55\tilde{C}^3R_{\text{GP-UCB}}(T)$ with $\tilde{C} > 1$
and its total computational complexity is given as $\otilde( |\aset|T\gamma^2(T) + \gamma^4(T))$.
Here we use the maximum information gain instead of the effective dimension since
they have the same order up to polylogarithmic factors \citep{calandriello2019gaussian}.
If the kernel is an SE kernel,
ignoring polylogarithmic factors, their
computational complexity is linear in $T$.
However, the term incurs generally large coefficient $|\aset|$ in the term
unlike \algpe{}.
Finally, we note that we construct \algpeh{}
from PHASED ELIMINATION \citep{lattimore2020learning},
which is an algorithm for the stochastic misspecified linear bandit problem,
where PE stands for PHASED ELIMINATION.


\section{Problem Formulation}
\label{sec:prob-form}
Let $\Omega$ be a non-empty subset of a Euclidean space $\RR^d$ and
$K: \Omega \times \Omega \rightarrow \RR$ be a symmetric, positive definite kernel on $\Omega$, i.e.,
$K(x, y) = K(y, x)$ for all $x, y \in \Omega$ and
for a pairwise distinct points $\{x_1, \dots, x_n\} \subseteq \Omega$, the kernel matrix
$(K(x_i, x_j))_{1 \le i, j \le n}$ is positive definite.
Examples of such kernels are
Rational Quadratic (RQ), Squared Exponential (SE), and Mat\'ern kernels defined as
\begin{math}
  K_{\rqk}(x, y) := \bl 1 + \frac{s^2}{2\mu l^2}  \br^{-\mu},
\end{math}
\begin{math}
  K_{\gsk}(x, y) := \exp\bl -\frac{s^2}{2 l^2} \br,
\end{math}
and
\begin{math}
  K_{\mtk}^{(\nu)}(x, y) := \frac{2^{1-\nu}}{\Gamma(\nu)} \bl \frac{s\sqrt{2\nu}}{l} \br^{\nu}
  K_{\nu}\bl \frac{s\sqrt{2\nu}}{l} \br
\end{math}.
where $s = \| x - y\|_2$ and $l > 0, \mu > d/2$, $\nu >0$ are parameters, and $K_{\nu}$ is
the modified Bessel function of the second kind.
As in the previous work \parencite{chowdhury2017kernelized}, we normalize kernel $K$ so that
$K(x, x) \le 1$ for all $x \in \Omega$.
We note that the above three examples satisfy $K(x, x) = 1$ for any $x$.
We denote by $\hilb{\Omega}$ the RKHS corresponding to the kernel $K$,
which we shall review briefly in \S \ref{sec:res-approximation}
and assume that $f \in \hilb{\Omega}$ has bounded norm, i.e., $\|f \|_{\hilb{\Omega}} \le B$.
In this paper, we consider the following multi-armed bandit problem with time interval $T$
and arm set $\aset \subseteq \Omega$.
First, we formulate the stochastic RKHS bandit problem.
In each round $t = 1, 2, \dots, T$, a learner selects an arm $x_t \in \aset$ and
observes noisy reward $y_t = f(x_t) + \varepsilon_t$.
Here we assume that noise stochastic process is conditionally $R$-sub-Gaussian with respect to
a filtration $\{\mathcal{F}_t\}_{t=1,2,\dots}$, i.e.,
\begin{math}
  \ex{\exp(\xi \varepsilon_t) \mid \mathcal{F}_{t}} \le \exp(\xi^2 R^2/2)
\end{math}
for all $t\ge 1$ and $\xi \in \RR$. We also assume that $x_t$
is $\mathcal{F}_{t}$-measurable and $y_t$ is $\mathcal{F}_{t+1}$-measurable.
The objective of the learner is to maximize the cumulative reward
$\sum_{t=1}^{T}f(x_t)$ and regret is defined by
$R(T) := \sup_{x \in \aset}\sum_{t=1}^{T}\bl f(x) - f(x_t)\br$.
In the adversarial (or non-stochastic) bandit RKHS problem,
we assume a sequence $f_t \in \hilb{\Omega}$ with $\|f_t\|_{\hilb{\Omega}} \le B$ for $t = 1, \dots, T$
is given. In each round $t = 1,\dots, T$,
a learner selects an arm $x_t \in \aset$ and observes a reward $f_t(x_t)$.
The learner's objective is to minimize the cumulative regret
$R(T) := \sup_{x \in \aset} \sum_{t=1}^{T} f_t(x) - \sum_{t=1}^{T}f_t(x_t)$.
In this paper we only consider
oblivious adversary, i.e., we assume the adversary chooses a sequence $\{f_t\}_{t=1}^{T}$
before the game starts.

\section{Results from Approximation Theory}
In this section, we introduce important results provided by approximation theory.
For introduction to this subject, we refer to the monograph \parencite{wendland2004scattered}.
We first briefly review basic properties of the RKHS and
introduce classical results regarding the convergence rate of the power function,
which are required for the proof of Theorem \ref{thm:pgreedy-decay}.
Then, we introduce the $P$-greedy algorithm and
its convergence rate in Theorem \ref{thm:pgreedy-decay},
which generalizes the existing result \parencite{santin2017convergence}.

\label{sec:res-approximation}
\subsection{Reproducing Kernel Hilbert Space}
Let $F(\Omega) := \{f: \Omega \rightarrow \RR \}$ be the real vector space of
$\RR$-valued functions on $\Omega$.
Then, there exists a unique real Hilbert space
$\left(\hilb{\Omega}, \langle \cdot, \cdot \rangle_{\hilb{\Omega}}\right)$
with $\hilb{\Omega} \subseteq F(\Omega)$
satisfying the following two properties:
(i) $K(\cdot, x) \in \hilb{\Omega}$ for all $x \in \Omega$.
(ii) $\langle f, K(\cdot, x)\rangle_{\hilb{\Omega}} = f(x)$ for all $f \in \hilb{\Omega}$ and $x \in \Omega$.
Because of the second property,
the kernel $K$ is called reproducing kernel and $\hilb{\Omega}$ is called the reproducing kernel Hilbert
space (RKHS).

For a subset $\Omega' \subseteq \Omega$, we denote by $V(\Omega')$
the vector subspace of $\hilb{\Omega}$ spanned by $\{K(\cdot, x)\mid x \in \Omega' \}$.
We define an inner product of $V(\Omega')$ as follows.
For $f = \sum_{i \in I} a_i K(\cdot, x_i)$ and $g = \sum_{j \in I} b_j K(\cdot, x_j)$ with $|I| < \infty$,
we define $\langle f, g\rangle := \sum_{i, j\in I} a_i b_j K(x_i, x_j)$.
Since $K$ is symmetric and positive definite,
$V(\Omega')$ becomes a pre-Hilbert space with this inner product.
Then it is known that RKHS $\hilb{\Omega}$ is isomorphic to the completion of
$V(\Omega)$. Therefore, for each $f \in \hilb{\Omega}$,
there exists a sequence $\{x_n\}_{i=1}^{\infty}\subseteq \Omega$ and real numbers $\{a_n\}_{n=1}^{\infty}$
such that $f = \sum_{n=1}^{\infty}a_n K(\cdot, x_n)$. Here the convergence
is that with respect to the norm of $\hilb{\Omega}$ and because of a special property of RKHS,
it is also a pointwise convergence.

\subsection{Power Function and its Convergence Rate}
Since for any $f \in \hilb{\Omega}$, there exists a sequence of
finite sums $\sum_{n=1}^{N}a_n K(\cdot, x_n)$ that converges to $f$,
it is natural to consider the error between $f$ and the finite sum.
A natural notion to capture the error for any $f \in \hilb{\Omega}$ is
the power function defined as below.
For a finite subset of points $X = \{x_n\}_{n=1}^{N} \subseteq \Omega$,
we denote by $\Pi_{V(X)}: \hilb{\Omega} \rightarrow V(X)$ the orthogonal
projection to $V(X)$.
We note that the function $\Pi_{V(X)} f$ is characterized as the interpolant of $f$, i.e,
$\Pi_{V(X)} f$ is a unique function $g \in V(X)$ satisfying $g(x) = f(x)$ for all $x \in X$.
Then the power function $P_{V(X)}: \Omega \rightarrow \RR_{\ge 0}$ is defined as:
\begin{equation*}
  P_{V(X)}(x) = \sup_{f\in \hilb{\Omega} \setminus \{0 \}}
  \frac{
    | f(x) - (\Pi_{V(X)}f)(x) |
  }{
    \|f \|_{\hilb{\Omega}}
  }.
\end{equation*}
By definition, we have
\begin{equation*}
  \left|f(x) - \bl \Pi_{V(X)}f\br(x) \right| \le \|f\|_{\hilb{\Omega}} P_{V(X)}(x)
\end{equation*}
for any $f \in \hilb{\Omega}$ and $x \in \Omega$.

Since the power function $P_{V(X)}$ represents how well the space $V(X)$ approximates
any function in $\hilb{\Omega}$ with a bounded norm,
it is intuitively clear that the value of $P_{V(X)}$ is small if $X$ is a ``fine'' discretization of $\Omega$.
The fineness of a finite subset $X = \{x_1, \dots, x_N\} \subseteq \Omega$ can
be evaluated by the fill distance $h_{X, \Omega}$ of $X$ defined as
\begin{math}
  \sup_{x \in \Omega}  \min_{1 \le n \le N} \| x - x_n\|_2.
\end{math}
We introduce classical results on
the convergence rate of the power function as $h_{X, \Omega} \rightarrow 0$.
We introduce two kinds of these results: polynomial decay and exponential decay.%
\footnote{We note that more generalized results including in the case of conditionally positive definite kernels
and differentials of functions of RKHS
are proved \citep[Chapter 11]{wendland2004scattered}.}
Before introducing the results, we define smoothness of kernels.
\begin{dfn}
  \begin{enumerate}[label=(\roman*)]
    \item We say $(K, \Omega)$ has finite smoothness%
    \footnote{By abuse of notation, omitting $\Omega$,
    we also say ``$K$ has finite smoothness''.}
    with a smoothness parameter
    $\nu \in \frac{1}{2}\ZZ_{> 0}$,
    if $\Omega$ is bounded and satisfies an interior cone condition (see remark below),
    and satisfies either the following condition (a) or (b):
    (a) $K \in C^{2\nu} (\Omega^{\iota} \times \Omega^{\iota})$, and
    all the differentials of $K$ of order $2\nu$ are bounded on $\Omega \times \Omega$.
    Here $\Omega^{\iota}$ denotes the interior.
    (b) There exists $\Phi: \RR^d \rightarrow \RR$
    such that $K(x, y) = \Phi(x - y)$,
    $\nu + d/2 \in \ZZ$,
    $\Phi$ has continuous Fourier transformation $\hat{\Phi}$
    and $\hat{\Phi}(x) = \Theta ((1 + \|x \|_2^{2})^{-(\nu + d/2)})$ as $\|x\|_2 \rightarrow \infty$.

    \item We say $(K, \Omega)$ has infinite smoothness if
    $\Omega$ is a $d$-dimensional cube $\{x \in \RR^d: |x - a_0|_{\infty} \le r_0 \}$,
    $K(x, y) = \phi(\| x- y\|_2)$ with a function $\phi: \RR_{\ge 0} \rightarrow \RR$,
    and there exists a positive integer $l_0$ and a constant $M > 0$ such that
    $\varphi(r) := \phi(\sqrt{r})$ satisfies
    $|\frac{d^l \varphi}{dr^l}(r)| \le l! M^l$ for any $l \ge l_0$
    and $r \in \RR_{\ge 0}$.
  \end{enumerate}
\end{dfn}
\begin{rem}
   (i) Results introduced in this subsection depend on local polynomial reproduction on $\Omega$
   and such a result is hopeless if $\Omega$ is a general bounded set \parencite{wendland2004scattered}.
   The interior cone condition is a mild condition that assures such results.
   For example, if $\Omega$ is a cube
   $\{x: |x - a|_{\infty} \le r \}$ or
   ball $\{x: |x - a|_{2} \le r\}$, then this condition is satisfied.
  (ii) Since $\hat{\Phi}(x) = c \ (1 + \|x\|_2^2)^{-\nu -d/2}$ with $c > 0$
  for $\Phi(x) = \|x\|_2^{\nu}K_{\nu}(\|x \|_2)$,
  Mat\'ern kernels $K_{\mtk}^{(\nu)}$ have finite smoothness with smoothness parameter $\nu$.
  In addition, it can be shown that the RQ and SE kernels have infinite smoothness.
\end{rem}

\begin{thm}[\citet{wu1993local}, \citet{wendland2004scattered} Theorem 11.13]
  \label{thm:qunif-conv-pol}
  We assume $(K, \Omega)$ has finite smoothness with smoothness parameter $\nu$.
  Then there exist constants $C > 0$ and $h_0 > 0$ that depend
  only on $\nu, d, K$ and $\Omega$ such that
  \begin{math}
    \| P_{X} \|_{\linf{\Omega}} \le C h_{X, \Omega}^{\nu}
  \end{math}
  for any $X \subseteq \Omega$ with $h_{X, \Omega} \le h_0$.
\end{thm}

One can apply this result to RQ and SE kernels for any $\nu > 0$,
but a stronger result holds for these kernels.
\begin{thm}[\citet{madych1992bounds}, \citet{wendland2004scattered} Theorem 11.22]
  \label{thm:qunif-conv-exp}
  Let $\Omega \subset \RR^d$ be a cube
  and assume $K$ has infinite smoothness.
  Then, there exist constants $C_1, C_2, h_0 > 0$ depending only on
  $d, \Omega,$ and $K$ such that
  \begin{equation*}
    \| P_{X} \|_{L^{\infty}(\Omega)}
    \le
    C_1 \exp \bl  -C_2 /h_{X, \Omega} \br,
  \end{equation*}
  for any finite subset $X \subseteq \Omega$ with $h_{X, \Omega} \le h_0$.
\end{thm}
\begin{rem}
  (i) The assumption on $\Omega$ can be relaxed, i.e.,
  the set $\Omega$ is not necessarily a cube. See \parencite{madych1992bounds} for details.
  (ii) In the case of SE kernels, a stronger result holds.
    More precisely, for sufficiently small $h_{X, \Omega}$,
    $\| P_{X} \|_{L^{\infty}(\Omega)}
    \le C_1 \exp \bl  C_2 \log(h_{X, \Omega}) /h_{X, \Omega} \br$
    holds.
\end{rem}

\subsection{$P$-greedy Algorithm and its Convergence Rate}
\begin{algorithm}[tb]
  \caption{Construction of Newton basis with $P$-greedy algorithm (c.f. \citet{pazouki2011bases})}
\begin{algorithmic}
  \label{alg:nbasispgreedy}
  \Input{
    kernel $K$,
    admissible error $\admerr > 0$,
    a subset of points $\omdsc \subseteq \Omega$,
  }
  \Output{
    A subset of points $X_m \subseteq \omdsc$ and
    Newton basis $N_1, \dots, N_m$ of $V(X_m)$.}
  \STATE $\xi_1 := \argmax_{x \in \omdsc} K(x, x)$.
  \STATE $N_1(x) := \frac{K(x, \xi_1)}{\sqrt{K(\xi_1, \xi_1)}}$.
  \FOR{$m = 1, 2, 3, \dots,$}
    \STATE $P_{m}^2 (x) := K(x, x) - \sum_{k=1}^{m}N_k^2(x)$.
    \IF{$\max_{x \in \omdsc}P_{m}^2(x) < \admerr^2$}
      \STATE \textbf{return} $\{\xi_1, \dots, \xi_m\}$ and $\{N_1, \dots, N_m\}$.
    \ENDIF
    \STATE $\xi_{m + 1} := \argmax_{x \in \omdsc} P_{m}^2 (x)$.
    \STATE $u(x) := K(x, \xi_{m + 1}) - \sum_{k=1}^{m}N_k(\xi_{m + 1}) N_k(x)$,
    $N_{m + 1}(x) := u(x) / \sqrt{P_m^2(\xi_{m + 1})}$.
  \ENDFOR
\end{algorithmic}
\end{algorithm}

In a typical application, for a given discretization $\omdsc \subseteq \Omega$
and function $f \in \hilb{\Omega}$,
we want to find a finite subset $X = \{\xi_1, \dots, \xi_m\} \subseteq \omdsc$ with $|X| \ll |\omdsc|$
so that $f$ is close to an element of $V(X)$.
Several greedy algorithms are proposed to solve this problem
\parencite{de2005near,schaback2000adaptive,smuller2009komplexitaet}.
Among them, the $P$-greedy algorithm \parencite{de2005near} is most suitable
for our purpose, since the point selection depends only on $K$ and $\omdsc$ but not on the
function $f$ which is unknown to the learner in the bandit setting.

The $P$-greedy algorithm first selects a point $\xi_1 \in \omdsc$ maximizing $P_{V(\emptyset)}(x) = K(x, x)$
and after selecting points $X_{n - 1} =\{\xi_1, \dots, \xi_{n - 1}\}$, it selects
$\xi_n$ by $\xi_n = \argmax_{x \in \omdsc} P_{V(X_{n -1})}(x)$.
Following \myciteauthoryear{pazouki2011bases},
we introduce (a variant of) the $P$-greedy algorithm that simultaneously computes
the Newton basis \parencite{muller2009newton} in Algorithm \ref{alg:nbasispgreedy}.
If $\omdsc$ is finite, this algorithm outputs Newton Basis $N_1, \dots, N_m$ at the cost of
$O(|\omdsc|m^2)$ time complexity using $O(|\omdsc|m)$ space.
Newton basis $\{N_1, \dots, N_m\}$ is the Gram-Schmidt orthonormalization
of basis $\{K(\cdot, \xi_1), \dots, K(\cdot, \xi_m)\}$.
Because of orthonormality, the following equality holds
\citep[Lemma 5]{santin2017convergence}:
\begin{math}
  P_{V(X)}^2(x) = K(x, x) - \sum_{i = 1}^{m} N_i^2(x),
\end{math}
where $X = \{\xi_1, \dots, \xi_m\}$.
Seemingly, Algorithm \ref{alg:nbasispgreedy} is different from
the $P$-greedy algorithm described above, using this formula,
we can see that these two algorithms are identical.

The following theorem is essentially due to \myciteauthoryear{santin2017convergence}
and we provide a more generalized result.
\begin{thm}[\citet{santin2017convergence}]
  \label{thm:pgreedy-decay}
  Let $K: \Omega \times \Omega \rightarrow \RR$ be a symmetric positive definite kernel.
  Suppose that the $P$-greedy algorithm applied to $\omdsc \subseteq \Omega$ with
  error $\admerr$ gives
  $n_{\admerr}$ points $X \subseteq \omdsc$ with $|X| = n_{\admerr}$.
  Then the following statements hold:

  (i) Suppose $(K, \Omega)$ has finite smoothness with smoothness parameter
    $\nu > 0$.
    Then, there exists a constant $\hat{C} > 0$ depending only on $d, \nu, K$, and $\Omega$
    such that
    \begin{math}
      \| P_{V(X)} \|_{L^{\infty}\bl\omdsc\br} < \hat{C} n_{\admerr}^{-\nu/d}.
    \end{math}

    (ii) Suppose $(K, \Omega)$ has infinite smoothness.
    Then there exist constants $\hat{C}_1, \hat{C}_2 > 0$ depending only on
    $d, K,$ and $\Omega$ such that
     $\|P_{V(X)} \|_{L^{\infty}\bl \omdsc \br}
    < \hat{C_1} \exp\left(-\hat{C_2} n_{\admerr}^{1/d} \right)$.
\end{thm}
The statements of the theorem are non-trivial in two folds.
First, by Theorems \ref{thm:qunif-conv-pol}, \ref{thm:qunif-conv-exp},
if $n \in \ZZ_{> 0}$ is sufficiently large,
there exists a subset $X \subseteq \Omega$ with $|X| = n$
that gives the same convergence rate above (e.g. $X$ is a uniform mesh of $\Omega$).
This theorem assures the same convergence rate is achieved by the points selected by
the $P$-greedy algorithm.
Secondly, it also assures that the same result holds even if the $P$-greedy algorithm
is applied to a subset $\omdsc \subseteq \Omega$.

If the kernel has finite smoothness,
\myciteauthoryear{santin2017convergence} only considered the case when
$\hilb{\Omega}$ is norm equivalent to a Sobolev space,
which is also norm equivalent to the RKHS associated with a Mat\'ern kernel.
One can prove
Theorem \ref{thm:pgreedy-decay} from Theorems \ref{thm:qunif-conv-pol},
\ref{thm:qunif-conv-exp} and \citep[Corollary 3.3]{devore2013greedy}
by the same argument to \myciteauthoryear{santin2017convergence}.

For later use, we provide a restatement of Theorem \ref{thm:pgreedy-decay} as follows.
\begin{cor}
  \label{cor:pgreedy-decay}
    Let $\alpha, q > 0$ be parameters and
    denote by $D = D_{q, \alpha}(T)$ the number
    of points returned by the $P$-greedy algorithm
    with error $\admerr = \alpha/ T^q$.

    (i) Suppose $K$ has finite smoothness with smoothness parameter
    $\nu > 0$.
    Then
    $D_{q, \alpha}(T) = O\left(\alpha^{-d/\nu} T^{dq/\nu}\right)$.

    \noindent
    (ii) Suppose $K$ has infinite smoothness.
    Then
    $D_{q, \alpha}(T) = O\left((q\log T - \log(\alpha))^d\right)$.
\end{cor}

\section{Misspecified Linear Bandit Problem}
\label{sec:miss-spec}
Since we can approximate $f \in \hilb{\Omega}$ by an element of $V(X)$,
where $V(X)$ is a finite dimensional subspace of the RKHS,
we study a linear bandit problem where the linear model is misspecified,
i.e, the misspecified linear bandit problem \citep{lattimore2020learning,lattimore2020book}.
In this section, we introduce several algorithms for the stochastic and adversarial misspecified linear bandit problems.
It turns out that such algorithms can be constructed by modifying
(or even without modification) algorithms for the linear bandit problem.
We provide proofs in this section in the supplementary material.

First, we provide a formulation of the stochastic misspecified linear bandit problem
suitable for our purpose.
Let $\aset$ be a set and suppose that there exists a
map $x \mapsto \wx$ from $\aset$ to the unit ball
$\{\xi \in \RR^D: \|\xi \|_2 \le 1\}$ of a Euclidean space.
In each round $t = 1, 2, \dots, T$,
a learner selects an action $x_t \in \aset$ and the environment
reveals a noisy reward $y_t = g(x_t) + \varepsilon_t$,
where
$g(x) := \langle \theta, \wx\rangle + \omega(x)$, $\theta \in \RR^D$, and
$\omega(x)$ is a biased noise and satisfies
$\sup_{x \in \aset} \|\omega(x)\| \le \err$ and $\err > 0$ is known to the learner.
We also assume that there exists $B > 0$ such that
$\sup_{x \in \aset}\|g(x)\| \le B$ and $\| \theta\|_2 \le B$.
As before,
$\{\varepsilon_t\}_{t \ge 1}$ is conditionally $R$-sub-Gaussian
w.r.t a filtration $\{\filt_t\}_{t \ge 1}$ and we assume that
$\wx_t$ is $\filt_t$-measurable and $y_{t}$ is $\filt_{t + 1}$-measurable.
The regret is defined as
\begin{math}
  R(T) := \sum_{t=1}^{T}
  \bl \sup_{x \in \aset} g(x) - g(x_t) \br.
\end{math}
We can formulate the adversarial misspecified linear bandit problem in a similar way.
Let $\{g_t\}_{t=1}^{T}$ be a sequence of functions on $\aset$
with $g_t (x) = \langle \theta_t, \wx \rangle + \omega_t(x)$,
$\theta_t \in \RR^{D}$, and
$\sup_{x \in \aset} \| \omega_t(x)\| \le \err$,
where the map $x \mapsto \wx$ is as before.
We also assume that there exists $B > 0$ such that
$\sup_{x \in \aset}|g(x)| \le B$ and $\| \theta_t\|_2 \le B$.
In each round, $t = 1, \dots, T$, the learner selects an arm $x_t \in \aset$ and
observes a reward $g_t(x_t)$.
The cumulative regret is defined as
$R_T = \sup_{x \in \aset}\sum_{t=1}^{T}g_t(x) - \sum_{t=1}^{T}g_t(x_t)$.

First, we introduce a modification of LinUCB \parencite{abbasi2011improved}.
To do this, we prepare notation for the stochastic linear bandit problem.
Let $\lambda > 0$ and $\delta > 0$ be parameters.
We define
$A_t := \lambda 1_D + \sum_{s=1}^{t} \wx_s \wx_s^\trn$,
$b_t := \sum_{s=1}^{t} y_s \wx_s$, and $\thetahat_t := A_t^{-1} b_t$.
Here, $1_D$ is the identity matrix of size $D$.
For $x \in \RR^{D}$, we define the Mahalanobis norm as
$\|x\|_{A_{t}^{-1}} := \sqrt{x^{\trn} A_{t}^{-1} x}$
and define $\beta_t$ as
\begin{equation*}
  \beta_t := \beta(A_t, \delta, \lambda) := R \sqrt{\log \frac{\det \lambda^{-1} A_t }{\delta^2}} + \sqrt{\lambda} B.
\end{equation*}
We note that by the proof of \citep[Lemma 11]{abbasi2011improved},
computational complexity for updating $\beta_t$
is $O(D^2)$ at each round.

\myciteauthoryear{lattimore2020learning} (see appendix of its arXiv version)
considered a modification of LinUCB which selects
$x \in \aset$ maximizing (modified) UCB
$\langle \thetahat_t, \wx \rangle + \beta_t \|\wx\|_{A_t^{-1}}
+ \err \sum_{s=1}^{t}
|\wx^\trn A_t^{-1} \wx_s|$
in $(t + 1)$th round and
proved the regret of the algorithm is upper bounded by $O(D\sqrt{T} \log(T) + \err T \sqrt{D \log(T)})$.
However, computing the above value requires $O(t)$ time for each arm $x \in \aset$.
Therefore, instead of incurring additional $\sqrt{D}$ factor in the second term
in the regret upper bound above,
we consider another upper confidence bound which can be easily computed.
In $(t + 1)$th round, our modification of UCB type algorithm selects arm $x \in \aset$ maximizing the modified UCB
\begin{math}
  \langle \thetahat_t, \wx\rangle +
  \|\wx\|_{A_t^{-1}}\left(
  \beta_t + \err \psi_t
  \right),
\end{math}
where $\psi_t$ is defined as $\sum_{s=1}^{t} \| \wx_s \|_{A_{s-1}^{-1}}$.
Then by storing $\psi_t$ in each round,
the complexity for computing this value is given as
$O(D^2)$ for each $x \in \aset$ and
as is well-known, one can update $A_t^{-1}$ in $O(D^2)$ time using the Sherman–Morrison formula.
By the standard argument, we can prove the following regret bound.
\begin{prop}
  \label{prop:app-ucb}
  Let notation and assumptions be as above.
  We further assume that $\lambda \ge 1$.
  Then with probability at least $1 - \delta$,
  the regret $R(T)$ of the modified UCB algorithm satisfies
  \begin{math}
    R(T) \le 2\beta_T \sqrt{T} \sqrt{2 \log \det (\lambda^{-1}A_t)}
    \allowbreak + 2\err T \allowbreak +
    4\err T \log(\det(\lambda^{-1}A_T)).
  \end{math}
  In particular, we have
  \begin{equation*}
    R(T) = \otilde\left(\sqrt{DT\log(1/\delta)} + D\sqrt{T} +\err D T\right).
  \end{equation*}
\end{prop}
In the supplementary material, we also introduce a modification of Thompson Sampling.


The regret upper bound provided above does not depend on the arm set $\aset$.
Moreover, the same results hold even if the arm set changes over time step $t$
(with minor modification of the definition of regret).
On the other hand, several authors \parencite{lattimore2020learning,auer2002using,valko2013finite}
studied algorithm whose regret depends on the cardinality of the
arm set in the stochastic linear or RKHS setting.
In some rounds, such algorithms eliminate arms that are supposed to be non-optimal
with a high probability
and therefore the arm set should be the same over time.
Generally, these algorithms are more complicated than LinUCB or Thompson Sampling.
However, recently, \myciteauthoryear{lattimore2020learning} proposed
a simpler and sophisticated algorithm
called PHASED ELIMINATION using Kiefer–Wolfowitz theorem.
Furthermore, they showed that it works well
for the stochastic misspecified linear bandit problem without modification.
More precisely, they proved the following result.
\begin{thm}[\citet{lattimore2020learning,lattimore2020book}]
  \label{thm:lmb-pe}
  Let $R(T)$ be the regret PHASED ELIMINATION incurs for the stochastic misspecified linear bandit problem.
  We further assume that $\{\varepsilon_t\}$ is independent $R$-sub-Gaussian.
  Then, with probability at least $1 - \delta$, we have
  \begin{equation*}
    R(T) = O\left(\sqrt{DT \log \left( \frac{|\aset| \log(T)}{\delta}\right)} +
    \err \sqrt{D} T \log(T) \right).
  \end{equation*}
  Moreover the total computational complexity up to round $T$
  is given as $O( D^2 |\aset| \log \log (D) \log(T) + TD^2)$.
\end{thm}
\begin{rem}
  Although they provided an upper bound for the expected regret,
  it is not difficult to see that their proof gave a high probability regret upper bound.
\end{rem}

Next, we show that EXP3 for adversarial linear bandits (c.f. \citet{lattimore2020book})
works for the adversarial misspecified linear bandits without modification.
We introduce notations for EXP3.
Let $\eta>0$ be a learning rate,
$\gamma$ an exploration parameter,
and $\piexp$ be an exploration distribution
over $\aset$.
For a distribution $\pi$ on $\aset$, we define a matrix
$Q(\pi) := \sum_{x \in \aset} \pi(x) \wx \wx^{\trn}$.
We also put $\phi_t := g_t(x_t) Q_t^{-1}\wx_t$ and
$\phi_t(x) := \langle \phi_t, \wx\rangle$ for $x \in \aset$, where
the matrix $Q_t$ is defined later.
We define a distribution $q_t$ over $\aset$
by $q_t(x) \sim \exp(\eta \sum_{s=1}^{t-1}\phi_s(x))$
and a distribution $p_t$ by $p_t(x) = \gamma \piexp(x) + (1-\gamma)q_t(x)$
for $x \in \aset$.
The matrix $Q_t$ is defined as $Q(p_t)$.
We put $\Gamma(\piexp) := \sup_{x \in \aset}\wx^\trn Q(\piexp)^{-1}\wx$.
\begin{prop}
  \label{prop:appexp3}
  We assume that $\{\wx \mid x \in \aset \}$ spans $\RR^D$.
  We also assume $\piexp$ satisfies $\Gamma(\piexp) \le D$
  and we take $\gamma = B\Gamma(\piexp)\eta$.
  Then applying EXP3 to the adversarial misspecified linear bandit problem,
  we have the following upper bound for the expected regret:
  \begin{equation*}
    \ex{R(T)} \le 2\varepsilon T +
    e B^2 \eta DT + \frac{2 \err T } {B\eta} + \frac{\log | \aset |}{\eta}.
  \end{equation*}
\end{prop}
\begin{rem}
By the Kiefer–Wolfowitz theorem, there exists an exploitation distribution
$\piexp$ such that $\Gamma(\piexp) \le D$.
\end{rem}

\section{Main Results}
\label{sec:main-results}
Using results from approximation theory explained in \S \ref{sec:res-approximation}
and algorithms for the misspecified bandit problem,
we provide several algorithms for the stochastic and adversarial RKHS bandit problems.
We provide proofs of the results in this section in the supplementary material.

Let $N_1, \dots, N_D$ be the Newton basis returned by Algorithm \ref{alg:nbasispgreedy}
with $\admerr = \frac{\alpha}{T^q}$ with $q, \alpha > 0$, and $\omdsc = \aset$.
Then, by orthonormality of the Newton basis and the definition of the power function,
for any $f \in \hilb{\Omega}$ and $x \in \Omega$,
we have
\begin{equation*}
  |f(x) - \langle \theta_f, \wx \rangle| \le \| f \|_{\hilb{\Omega}} P_{V(X)}(x),
\end{equation*}
where $\theta_f = \bl \langle f, N_i \rangle\br_{1\le i \le D} \in \RR^D$
and $\wx = \bl N_i(x) \br_{1\le i \le D} \in \RR^D$.
Therefore,
if $f$ is an objective function of a RKHS bandit problem,
we can regard $f$ as a linearly misspecified model
and apply algorithms for misspecified linear bandit problems
to solve the original RKHS bandit problems.

In this section,
we reduce the RKHS bandit problem to the the misspecified linear bandit problem by the map
$x \mapsto \wx$ and
apply modified LinUCB, PHASED ELIMINATION, and EXP3 to the problem.
We call these algorithms \algucbh{}, \algpeh{} and \algexp{} respectively
and \algucbh{} is displayed in Algorithm \ref{alg:appucb}.
We denote by $D_{q, \alpha}(T) = D$ the number of points returned by Algorithm \ref{alg:nbasispgreedy}
with $\admerr = \frac{\alpha}{T^q}$.
By the results in \S \ref{sec:res-approximation},
we have an upper bound of $D_{q, \alpha}(T)$ (Corollary \ref{cor:pgreedy-decay}).

\begin{algorithm}[tb]
  \caption{Approximated RKHS Bandit Algorithm of UCB type (APG-UCB)}
  \begin{algorithmic}
  \label{alg:appucb}
  \Input{
    Time interval $T$, admissible error $\admerr = \frac{\alpha}{T^q}$, $\lambda, R, B, \delta$
  }
  \STATE
  Using Alg. \ref{alg:nbasispgreedy},
  compute Newton basis $N_1, \dots, N_D$ with admissible error $\admerr$
  and $\omdsc = \aset$, and put $\err = B \admerr$.
  \FOR{$x \in \aset$}
    \STATE $\wx := [N_1(x), N_2(x), \dots, N_D(x)]^\trn \in \RR^D$.
  \ENDFOR
  \FOR{$t = 0,1, \dots, T-1$}
    \STATE $A_t := \lambda 1_D + \sum_{s = 1}^{t} \wx_s \wx_s^\trn$, \
    $b_t := \sum_{s = 1}^{t} y_s \wx_s$.
    \STATE $\thetahat_t := A_t ^{-1} b_t$,\quad
    $\psi_t := \sum_{s=1}^{t} \|\wx_s \|_{A_{s-1}^{-1}}$.
    \STATE $x_{t + 1} :=
    \argmax_{x \in \aset} \left\{
    \langle \wx, \thetahat_t \rangle + \| \wx \|_{A_t^{-1}}
    \bl \beta_t + \err \psi_t\br \right\}$.
    \STATE Select $x_{t+1}$ and observe $y_{t+1}$.
  \ENDFOR
  \end{algorithmic}
\end{algorithm}

First, we state the results for \algucbh{}.
\begin{thm}
  \label{thm:app-ucb}
  We denote by $R_{\algucb}(T)$ the regret that Algorithm \ref{alg:appucb} incurs
  for the stochastic RKHS bandit problem up to time step $T$
  and assume that $\lambda \ge 1$ and $q \ge 1/2$.
  Then with probability at least $1 - \delta$,
  $R_{\algucb}(T)$ is given as
  \begin{equation*}
    \otilde\left(
    \sqrt{T D_{q, \alpha}(T) \log(1/\delta)} + D_{q, \alpha}(T) \sqrt{T}
    \right)
  \end{equation*}
  and the total
  computational complexity of the algorithm is given as $O(|\aset|T D_{q, \alpha}^2(T))$.
\end{thm}
The admissible error $\admerr$ balances the computational complexity and
regret minimization.
However, this is not clear from Theorem \ref{thm:app-ucb}.
The following theorem provides another upper bound of \algucbh{}
and it states that if we take smaller error $\admerr$,
then an upper bound of \algucbh{} is almost the same as that of IGP-UCB.
\begin{thm}
  \label{thm:igp-app}
  We assume $\lambda = 1$
  and take parameter $q$ of \algucbh{} so that $q > 3/2$.
  We define $\betaigp_T$ as $B + R\sqrt{2(\gamma_T + 1 + \log(1/\delta))}$.
  Then with probability at least $1-\delta$, we have
  $R_{\algucb}(T) \le b(T)$, where $b(T)$ is given as
  \begin{math}
    4 \betaigp_T \sqrt{\gamma_T T}
     + O(\sqrt{T\gamma_T} T^{(3/2-q)/2}
    + \gamma_T T^{1-q}).
  \end{math}
\end{thm}
\begin{rem}
  Since the main term of $b(T)$ is $4 \betaigp_T \sqrt{\gamma_T T}$
  and by the proof in \citep{chowdhury2017kernelized},
  IGP-UCB has the regret upper bound $4 \betaigp_T \sqrt{\gamma_T (T + 2)}$,
  \algucbh{} has an asymptotically the same
  regret upper bound as IGP-UCB if we take a small error $\admerr$.
  We note that if $\nu$ is sufficiently large compared to $d$
  (this is always the case if the kernel has infinite smoothness),
  then \algucbh{} is more efficient than IGP-UCB.
  We note that for any choice of parameters, the regret upper bound of BBKB is
  given as $55\tilde{C}^3 R_{\text{GP-UCB}}(T)$, where $\tilde{C} \ge 1$.
\end{rem}

Next, we state the results for \algpeh{}.
\begin{thm}
  \label{thm:app-pe}
  We denote by $R_{\algpe}(T)$ the regret that \algpeh{} with $q = 1/2$ incurs
  for the stochastic RKHS bandit problem
  up to time step $T$.
  We further assume that $\{\varepsilon_t\}$ is independent $R$-sub-Gaussian.
  Then with probability at least $1 - \delta$, we have
  \begin{math}
      R_{\algpe}(T) =
      \otilde\left(
      \sqrt{T D_{1/2, \alpha}(T)  \log\left(\frac{|\aset|}{\delta}\right)
      }
      \right),
  \end{math}
  and its total computational complexity is given as
  $\otilde \left( (|\aset| + T) D_{1/2, \alpha}^2(T) \right)$.
\end{thm}

Finally, we state a result for the adversarial RKHS bandit problem.
\begin{thm}
  \label{thm:app-exp3}
  We denote by $R_{\text{\algexp{}}}(T)$ the cumulative regret that
  \algexp{} with $\alpha = \log(|\aset|)$ and
  $q = 1$ incurs for the adversarial RKHS bandit problem
  up to time step $T$.
  Then with appropriate choices of the learning rate $\eta$ and
  exploration distribution, the expected regret
  $\ex{R_{\text{\algexp}}(T)}$ is given as $
      \otilde\left(
      \sqrt{T D_{1, \alpha}(T)  \log\left(|\aset|\right) }
      \right)$.
\end{thm}

\newcommand{\figureheight}{2.3cm}
\begin{figure*}[ht]
  \vskip 0.2in
  \centering
  \includegraphics[width=0.95\textwidth, height=\figureheight{}]{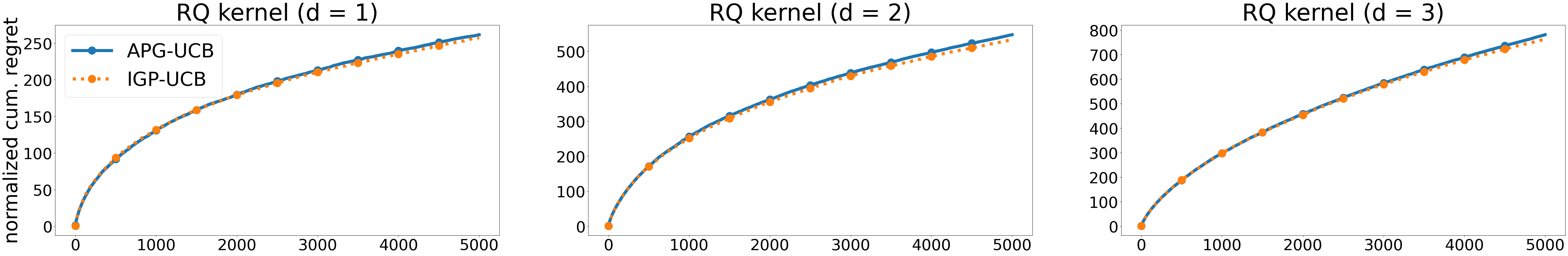}
  \caption{Normalized Cumulative Regret for RQ kernels.}
  \vskip -0.2in
  \label{fig:regret-rqkern}
\end{figure*}
\begin{figure*}[ht]
  \vskip 0.2in
  \centering
  \includegraphics[width=0.95\textwidth, height=\figureheight{}]{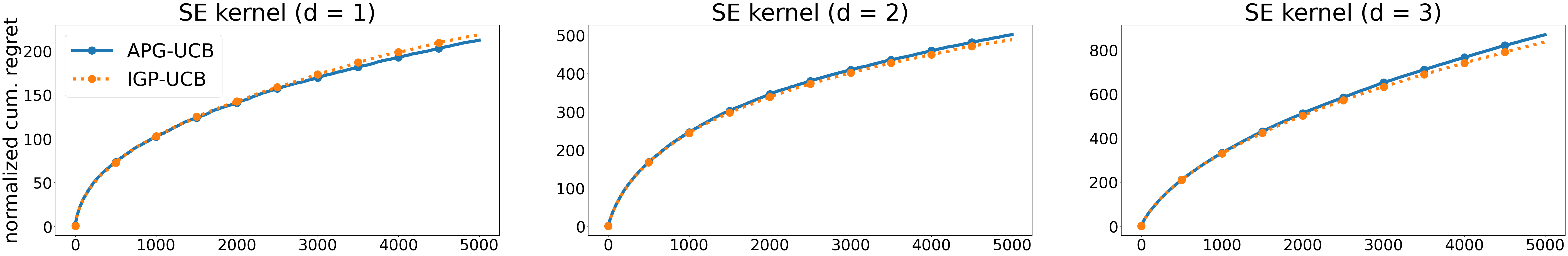}
  \caption{Normalized Cumulative Regret for SE kernels.}
  \label{fig:regret-sekern}
  \vskip -0.2in
\end{figure*}

\section{Discussion}
So far, we have emphasized the advantages of our methods.
In this section, we discuss limitations of our methods.
Here, we focus on Theorem \ref{thm:app-ucb} with $q = 1/2$ and Theorem \ref{thm:app-pe}.
Since we do not see limitations if the kernel has infinite smoothness,
in this section we assume the kernel is a Mat\'ern kernel.
In our theoretical results, $D_{q, \alpha}(T)$ plays a similar role as the
information gain in the theoretical result of BBKB.
If the kernel is a Mat\'ern kernel with parameter $\nu$,
then, by recent results on the information gain \citep{vakili2021information},
we have $\gamma_T = \otilde(T^{d/(d + 2\nu)})$,
which is a nearly optimal result by \citep{scarlett2017lower}
and is slightly better than the upper bound of $D_{1/2, \alpha}(T)$.
Therefore, in this case the regret upper bound $\otilde(\sqrt{T} T^{d/(2\nu)})$ of
Theorem \ref{thm:app-ucb} is slightly worse than
the regret upper bound $\otilde(\sqrt{T}T^{d/(d + 2\nu)})$ of BBKB.
In addition, similarly SupKernelUCB has nearly optimal regret upper bound
if the kernel is Mat\'ern, but regret upper bound of \algpeh{}
is slightly worse in that case.

Inferiority of our method in the Mat\'ern kernel case might be counter-intuitive since it is also proved
that the convergence rate of the power function
for Mat\'ern kernel is optimal (c.f. \citet{schaback1995error})
and Theorem \ref{thm:lmb-pe} cannot be improved \citep{lattimore2020learning}.
We explain why a combination of optimal results leads to a non-optimal result.
The results on the information gain
depend on the eigenvalue decay of the Mercer operator
rather than the decay of the power function in the $L^\infty$-norm
as in this study. However these two notions are closely related.
From the $n$-width theory \citep[Chapter IV, Corollary 2.6]{pinkus2012n},
eigenvalue decay corresponds to the decay of the power function in the $L^2$-norm
(or more precisely Kolmogorov $n$-width).
The decay in the $L^2$-norm is derived from that in the $L^{\infty}$-norm.
If the kernel is a Mat\'ern kernel, using a localization trick
called Duchon's trick \citep{wendland1997sobolev}, it can be possible
to give a faster decay in the $L^p$-norm than that in the $L^\infty$-norm
if $p < \infty$.
Since the norm regarding the misspecified bandit problem is not a $L^{2}$
norm but a $L^{\infty}$ norm, we took the approach proposed in this paper.

\section{Experiments}
In this section, we empirically verify our theoretical results.
We compare \algucbh{} to IGP-UCB \parencite{chowdhury2017kernelized}
in terms of cumulative regret and running time
for RQ and SE kernels in synthetic environments.
\subsection{Environments}
We assume the action set is a discretization of a cube $[0, 1]^d$
for $d = 1, 2, 3$.
We take $\aset$ so that $|\aset|$ is about $1000$.
More precisely, we define $\aset$ by
$\{i/m_d \mid i = 0, 1, \dots, m_d - 1\}^d$
where $m_1 = 1000, m_2 = 30, m_3 = 10$.
We randomly construct reward functions $f \in \hilb{\Omega}$ with $\| f\|_{\hilb{\Omega}} = 1$ as follows.
We randomly select points $\xi_i$ (for $1\le i \le m$) from $\aset$
until $m \le 300$ or $\|P_{V(\{\xi_1, \dots, \xi_m\})}\|_{L^\infty(\aset)} < 10^{-4}$
and compute orthonormal basis $\{\varphi_1, \dots, \varphi_m\}$
of $V(\{\xi_1, \dots, \xi_m\})$.
Then, we define $f = \sum_{i=1}^{m}a_i \varphi_i$, where
$[a_1, \dots, a_m] \in \RR^m$ is a random vector with unit norm.
We take $l = 0.3 \sqrt{d}$ for the RQ kernel and
$l = 0.2\sqrt{d}$ for the SE kernel,
because the diameter of the $d$-dimensional cube is $\sqrt{d}$.
For each kernel, we generate $10$ reward functions as above
and evaluate our proposed method and the existing algorithm
for time interval $T = 5000$
in terms of mean cumulative regret and total running time.
We compute the mean of cumulative regret and
running time for these 10 environments.
We normalize cumulative regret so that
normalized cumulative regret of uniform random policy
corresponds to the line through origin with slope 1 in the figure.
For simplicity, we assume the kernel, $B$, and $R$ are
known to the algorithms.
For the other parameters, we use theoretically suggested ones for both \algucbh{} and IGP-UCB.
Computation is done by Intel Xeon E5-2630 v4 processor with 128 GB RAM.
In supplementary material, we explain the experiment setting in more detail and
provide additional experimental results.

\subsection{Results}
We show the results for normalized cumulative regret
in Figures \ref{fig:regret-rqkern} and \ref{fig:regret-sekern}.
As suggested by the theoretical results, growth of the cumulative regret
of these algorithms is $\otilde(\sqrt{T})$ ignoring a polylogarithmic factor.
Although, convergence rate of the power function of SE kernels is
slightly faster than that of RQ kernels (by remark of Theorem \ref{thm:qunif-conv-exp}),
empirical results of RQ kernels and SE kernels are similar.
In both cases, \algucbh{} has almost the same cumulative regret as that of IGP-UCB.

We also show (mean) total running time in Table \ref{tab:comp-time},
where we abbreviate APG-UCB as APG and IGP-UCB as IGP.
For all dimensions, it took about from five to six thousand seconds for IGP-UCB
to complete an experiment for one environment.
As shown by the table and figures,
running time of our methods is much shorter than that of IGP-UCB while
it has almost the same regret as IGP-UCB.

\begin{table}[h]
  \caption{Total Running Time (in seconds).}
  \centering
  \label{tab:comp-time}
  \begin{tabular}{lll|ll}
       &  APG(RQ) &  IGP(RQ) & APG(SE) & IGP(SE) \\\hline
    $d = 1$ &  4.2e-01 &  5.7e+03 & 4.0e-01 &  5.7e+03 \\
    $d = 2$ &  2.7e+00 &  5.1e+03 & 2.9e+00 &  5.1e+03 \\
    $d = 3$ &  3.0e+01 &  5.7e+03 & 4.3e+01 &  5.7e+03 \\
  \end{tabular}
\end{table}

\section{Conclusion}
By reducing the RKHS bandit problem to the misspecified linear bandit problem,
we provide the first general algorithm for the adversarial RKHS bandit problem
and several efficient algorithms for the stochastic RKHS bandit problem.
We provide cumulative regret upper bounds for them and empirically verify our theoretical results.

\section{Acknowledgement}
We would like to thank anonymous reviewers for suggestions that improved the paper.
We also would like to thank Janmajay Singh and Takafumi J. Suzuki for valuable comments
on the preliminarily version of the manuscript.

\onecolumn
\input{supplement}

\end{document}

%% file: supplement.tex
\appendix
\newcommand{\mywy}{\widetilde{y}}
\newcommand{\wrho}{\widetilde{\rho}}

\section*{Appendix}
In this appendix, we provide some results on Thompson Sampling based algorithms
in \S \ref{sec:add-ts},
proofs of the results in \S \ref{sec:proofs}, and
detailed experimental setting and additional experimental results in \S \ref{sec:appendix-exp}.

\section{Additional Results for Thompson Sampling}
\label{sec:add-ts}
\subsection{Misspecified Linear Bandit Problem}
We consider a modification of Thompson Sampling \citep{agrawal2013thompson}.
In $(t + 1)$th round, we sample $\mu_t$
from the multinomial normal distribution $\normdist\left(
\thetahat_t, \left(\beta(A_t, \delta/2, \lambda) + \err \psi_t\right)^2A_t^{-1}
\right)$,
and the modified algorithm selects $x \in \aset$ that
maximizes $\langle \mu_t, \wx\rangle$.
Then the following result holds.
\begin{prop}
  \label{prop:app-ts}
  We assume that $\lambda \ge 1$.
  Then,
  with probability at least $1-\delta$, the modification of Thompson Sampling algorithm
  incurs regret upper bound by
  \begin{equation*}
    \otilde \left(
    \sqrt{\log (|\aset|)}
  \left\{
    D\sqrt{T} + \sqrt{DT\log(1/\delta)} + \log(1/\delta) \sqrt{T} +
   \left(D T + T \sqrt{D\log(1/\delta)}\right)\err)
  \right\}
  \right)
  \end{equation*}
\end{prop}
We provide proof of the proposition in \S \ref{sec:proof-prop-ts}.
\subsection{Thompson Sampling for the Stochastic RKHS Bandit Problem}
We provide a result on a Thompson Sample based algorithm for the stochastic RKHS bandit problem.
\begin{thm}
  \label{thm:app-ts}
  We reduce the RKHS bandit problem to the misspecified linear bandit problem,
  and apply the modified Thompson Sampling introduced above with
  admissible error $\admerr = \frac{\alpha}{T^q}$ with $q \ge 1/2$.
  We denote by $R_{\algthmp}(T)$ its regret and assume that $\lambda \ge 1$.
  Then with probability at least $1 - \delta$,
  $R_{\algthmp}(T)$ is upper bounded by
  \begin{equation*}
    \otilde\left(
    \sqrt{\log (|\aset|)}
    \left(
    D_{q, \alpha}(T)\sqrt{T} +
    \sqrt{D_{q, \alpha}(T)T\log(1/\delta)} +
     \log(1/\delta) \sqrt{T})
    \right)
    \right).
  \end{equation*}
  The total computational complexity of
   the algorithm is given as $O(|\aset|TD_{q, \alpha}^2(T) + TD_{q, \alpha}^3(T))$.
\end{thm}
We provide proof of the theorem in \S \ref{sec:subsec-proof-thm-app-ucb}.

\section{Proofs}
\label{sec:proofs}
We provide omitted proofs in the main article and \S \ref{sec:add-ts}.
\subsection{Proof of Corollary \ref{cor:pgreedy-decay}}
For completeness, we provide a proof of corollary \ref{cor:pgreedy-decay}.
\begin{proof}
  For simplicity, we consider only the infinite smoothness case.
  We use the same notation as in Theorem \ref{thm:pgreedy-decay}
  and Algorithm \ref{alg:nbasispgreedy}.
  Denote by $D$ the number of points returned by the algorithm with error $\admerr = \alpha / T^q$.
  Since the statement of the corollary is obvious if $D=1$, we assume $D > 1$.
  Because the condition $\max_{x \in \omdsc}P_m(x) < \alpha / T^q$ is satisfied
  only when $m \ge D$,
  we have
  \begin{math}
    \alpha/T^q \le \max_{x \in \omdsc} P_{D-1}(x)
  \end{math}.
  If we run the algorithm with error $\admerr = \max_{x \in \omdsc}P_{D-1}(x) + \epsilon$
  with sufficiently small $\epsilon > 0$, then
  the algorithm returns $D - 1$ points.
  Therefore by the theorem and the inequality above,
  we have
  \begin{equation*}
    \alpha / T^q \le \max_{x \in \omdsc} P_{D-1}(x) <
    \hat{C_1} \exp\left(-\hat{C_2} (D-1)^{1/d} \right).
  \end{equation*}
  Ignoring constants other than $\alpha, T, q$, we have the assertion of the corollary.
\end{proof}

\subsection{Proof of Proposition \ref{prop:app-ucb}}
For symmetric matrices $P, Q \in \RR^{n \times n}$, we write $P \geq Q$ if and only if
$P - Q$ is positive semi-definite, i.e., $x^T (P- Q) x \ge 0$ for all $x \in \RR^{n}$.
For completeness, we prove the following elementary lemma.
\begin{lem}
  \label{lem:elem}
  Let $P, Q \in \RR^{n \times n}$ be symmetric matrices of size $n$
  and assume that $0 < P \leq Q$.
  Then we have $Q^{-1} \leq P^{-1}$.
\end{lem}
\begin{proof}
  It is enough to prove the statement for $U^\trn P U$ and $U^\trn Q U$ for some $U \in \gl_n(\RR)$,
  where $\gl_n(\RR)$ is the general linear group of size $n$.
  Since $P$ is positive definite, using Cholesky decomposition,
  one can prove that there exists $U \in \gl_n(\RR)$ such that
  $U^\trn P U = 1_n$ and $U^\trn Q U = \Lambda$ is a diagonal matrix.
  Then, the assumption implies that every diagonal entry of $\Lambda$ is greater than or equal to $1$.
  Now, the statement is obvious.
\end{proof}

Next, we prove that
$\langle \wx, \theta_t\rangle + \| \wx\|_{A_t^{-1}} (\beta_t + \err \psi_t)$
is a UCB up to a constant.
\begin{lem}
  \label{lem:upper-bd}
  We assume $\lambda \ge 1$.
  Then, with probability at least $1 - \delta$,
  we have
  \begin{equation*}
    | \langle \wx, \theta_t\rangle - \langle \wx , \theta \rangle |
    \le \| \wx \|_{A_t^{-1}} (\beta_t + \err \psi_t),
  \end{equation*}
  for any $t$ and $\wx \in \RR^D$.
\end{lem}
\begin{proof}
  By proof of \citep[Theorem 2]{abbasi2011improved}, we have
  \begin{align*}
    | \langle \wx, \theta_t\rangle - \langle \wx , \theta \rangle |
    &\le \| \wx \|_{A_t^{-1}} \left(
    \left\| \sum_{s=1}^t \wx_s (\varepsilon_s + \omega(x_s)) \right\|_{A_t^{-1}} +
    \lambda^{1/2} \| \theta\|
    \right)\\
    & \le \| \wx \|_{A_t^{-1}}
    \left (
    \left\| \sum_{s=1}^t
    \wx_s \varepsilon_s \right\|_{A_t^{-1}} + \lambda^{1/2} \| \theta\|
    + \err \sum_{s=1}^t \|\wx_s \|_{A_t^{-1}}
    \right).
  \end{align*}
  By the self-normalized concentration inequality \citep{abbasi2011improved},
  with probability at least $1 - \delta$, we have
  \begin{equation*}
    | \langle \wx, \theta_t\rangle - \langle \wx , \theta \rangle | \le \|\wx\|_{A_t^{-1}}
    \left(\beta_t + \err \sum_{s=1}^t \|\wx_s \|_{A_t^{-1}} \right).
  \end{equation*}
  Since $A_{s-1} \le A_{s}$ for any $s$, by Lemma \ref{lem:elem},
  we have $\sum_{s=1}^{t}\|\wx_s \|_{A_t^{-1}} \le \psi_t$.
  This completes the proof.
\end{proof}

\begin{proof}[Proof of Proposition \ref{prop:app-ucb}]
  We assume $\lambda \ge 1$.
  Let $x^{*} := \argmax_{x \in \aset} f(x)$
  and $\bl x_t \br_t$ be a sequence of arms selected by the algorithm.
  Denote by $E$ the event on which the inequality in Lemma \ref{lem:upper-bd} holds
  for all $t$ and $\wx$.
  Then on event $E$, we have
  \begin{align*}
    f(x^*) - f(x_t) &\le 2 \err + \langle \wx^{*}, \theta \rangle - \langle \wx_t, \theta \rangle\\
    & \le 2\err + \langle \wx^*, \theta_t \rangle + \| \wx^* \|_{A_t^{-1}} (\beta_t + \err \psi_t)
    -\langle \wx_t, \theta \rangle\\
    & \le 2\err + \langle \wx_t, \theta_t \rangle + \| \wx_t\|_{A_t^{-1}} (\beta_t + \err \psi_t)
    - \left(
    \langle \wx_t, \theta_t \rangle - \| \wx_t\|_{A_{t}^{-1}} (\beta_t + \err \psi_t)
    \right)\\
    & = 2 \err + 2 \| \wx_t\|_{A_{t}^{-1}} (\beta_t + \err \psi_t).
  \end{align*}
  Therefore, on event $E$,
  \begin{align*}
    R(T) &\le 2 \err T + 2 \beta_T \sum_{t=1}^{T} \| \wx_t\|_{A_{t-1}^{-1}} +
    2\err \sum_{t=1}^{T} \| \wx_t\|_{A_{t-1}^{-1}}\psi_t\\
    &\le
    2\err T + 2\beta_T \sqrt{T} \sqrt{\sum_{t=1}^{T} \|\wx_t\|^2_{A_{t-1}^{-1}}}
    + 2\err \psi_T\sum_{t=1}^{T} \| \wx_t\|_{A_{t-1}^{-1}}\\
    &=
    2\err T + 2\beta_T \sqrt{T} \sqrt{\sum_{t=1}^{T} \|\wx_t\|^2_{A_{t-1}^{-1}}}
    + 2\err \left(\sum_{t=1}^{T} \| \wx_t\|_{A_{t-1}^{-1}}\right)^2\\
    & \le
    2\err T + 2\beta_T \sqrt{T} \sqrt{\sum_{t=1}^{T} \|\wx_t\|^2_{A_{t-1}^{-1}}}
    + 2\err T \left(\sum_{t=1}^{T} \| \wx_t\|^2_{A_{t-1}^{-1}}\right).
  \end{align*}
  By assumptions, we have $\|\wx\|_{A_{t-1}^{-1}} \le \|\wx\|_{A_{0}^{-1}} = \lambda^{-1/2} \| \wx\|_{2} \le 1$
  for any $x \in \aset$.
  Therefore, by \citep[Lemma 11]{abbasi2011improved}, the following inequalities hold:
  \begin{equation}
    \label{eq:beta-upp}
    \sum_{t=1}^{T} \|\wx_t\|^2_{A_{t-1}^{-1}} \le 2 \log(\det(\lambda^{-1}A_t)) \le 2 D\log\left( 1+ \frac{T}{\lambda D}\right),
    \quad
    \beta_t \le R \sqrt{D\log\left( 1 + \frac{T}{\lambda D} \right) + 2 \log(1/\delta)} +
    \sqrt{\lambda} B.
  \end{equation}
  Thus, on event $E$, we have
  \begin{align*}
    R(T)
    &\le 2 \beta_T \sqrt{T}  \sqrt{2\log(\lambda^{-1} A_T)}
    + 2\err T + 4 \err T \log(\det(\lambda^{-1} A_T))\\
    &= \otilde \left(
    \err T + (\sqrt{D} + \sqrt{1/\delta}) \sqrt{DT}
    + \err D T
    \right)\\
    & =
    \otilde \left(
    D \sqrt{T} + \sqrt{DT\log(1/\delta)} + \err D T
    \right).
  \end{align*}
\end{proof}

\subsection{Proof of Proposition \ref{prop:appexp3}}
This proposition can be proved by adapting the standard proof of the adversarial linear bandit
problem \citep{lattimore2020book,bubeck2012regret}.
We recall notation for the adversarial misspecified linear bandit problem and EXP3.
Let $\aset \ni x \mapsto \wx \in \{\xi \in \RR^{D} : \| \xi \| \le 1\}$ be a map,
$\{g_t\}_{t=1}^{T}$ be a sequence of reward functions on $\aset$ such that
$g_t(x) = \langle \theta_t, \wx \rangle + \omega_t(x)$ for $x \in \aset$,
where $\theta_t \in \RR^D$, $\sup_{x \in \aset}|g_t(x)|, \| \theta_t\| \le B$,
and $\sup_{x \in \aset}|\omega_t(x)| \le \epsilon$.

Let $\gamma \in (0, 1)$ be an exploration parameter, $\eta > 0$ a learning rate,
and $\piexp$ an exploitation distribution over $\aset$.
For a distribution $\pi$ over $\aset$, we put $Q(\pi) = \sum_{x \in \aset}\pi(x)\wx \wx^{\trn}$.
We define $\phi_t = g_x(x_t)Q_t^{-1}\wx_t$ and $\phi_t(x) = \langle \phi_t, \wx \rangle$
for $x \in \aset$, where the matrix $Q_t$ can be computed from the past observations
at round $t$ and is defined later.
Let $q_t$ be a distribution over $\aset$ such that $q_t(x) \sim \exp\left(
  \eta \sum_{s=1}^{t-1} \phi_s(x)
\right)$ and put $p_t(x) = \gamma \piexp(x) + (1 - \gamma) q_t(x)$
for $x \in \aset$.
We assume that $Q(\piexp)$ is non-singular and define $Q_t = Q(p_t)$.
For a distribution $\pi$ over $\aset$, we
define $\Gamma(\pi) = \sup_{x \in \aset} \wx^{\trn} Q(\piexp)^{-1}\wx$
and in this section we assume $\Gamma(\pi) \le D$.

Let $x_{*} = \argmax_{x \in \aset}\sum_{t=1}^{T}g_t(x)$ be an optimal arm
and regret is defined as $R(T) = \sum_{t=1}^{T} g_t(x_{*}) - \sum_{t=1}^{T}g_t(x_t)$.
We have
\begin{align}
 \ex{R(T)} &=
 \ex{\sum_{t=1}^{T}\left(
  \langle \theta_t, \wx_{*}\rangle - \langle \theta_t, \wx_t \rangle
 \right)}
 + \ex{\sum_{t=1}^{T} \left(\omega_t(\wx_*) - \omega_t(\wx_t) \right)} \notag \\
 &\le 2 \epsilon T +
 \ex{\sum_{t=1}^{T}\left(
  \langle \theta_t, \wx_{*}\rangle - \langle \theta_t, \wx_t \rangle
 \right)} \label{eq:app-exp3-rt}.
\end{align}
We denote
by $\mathcal{H}_{t-1}$ the sigma field generated by $x_1, \dots, x_{t-1}$
and
by $\mathbf{E}_{t-1}$ the conditional expectation conditioned on $\mathcal{H}_{t-1}$.
We note that $p_t(x), q_t(x)$ for $x \in \aset$ and $Q_t$ are $\mathcal{H}_{t-1}$-measurable but
$\phi_t$ is not.
Then we have
\begin{align}
  &\ex{\sum_{t=1}^{T} \langle \theta_t, \wx_t\rangle} = \ex{\sum_{t=1}^{T} \ex[t-1]{\langle \theta_t, \wx_t \rangle}}
  = \ex{\sum_{t=1}^{T}
  \sum_{x \in \aset} p_t(x) {\langle \theta_t, \wx \rangle }} \notag\\
  &=\gamma \ex{\sum_{t=1}^{T} \sum_{x \in \aset}{\piexp(x)\langle \theta, \wx \rangle}} +
  (1 - \gamma) \ex{\sum_{t=1}^{T} \sum_{ x \in \aset}{ q_t(x)\langle \theta_t, \wx \rangle}}
  \notag
  \\
  & \ge -\gamma BT + (1 - \gamma)S.\label{eq:app-exp3-s}
\end{align}
Here we used $|\langle \theta_t, \wx \rangle| \le \| \theta_t \| \| \wx\| \le B$
and $S$ is defined as $\ex{\sum_{t=1}^{T} \sum_{x \in \aset} {q_t(x) \langle \theta_t, \wx \rangle}}$.
Since $\sum_{t=1}^{T} \langle \theta_t, \wx_{*} \rangle \le \gamma BT + (1 - \gamma)
\sum_{t=1}^{T}\langle \theta_t, \wx_{*} \rangle$,
by inequalities \eqref{eq:app-exp3-rt}, \eqref{eq:app-exp3-s},
we have
\begin{equation}
  \ex{R(T)}  \le
  2 \epsilon T + 2 \gamma BT +
  (1 - \gamma) \left(
    \sum_{t=1}^{T} \langle \theta_t, \wx_{*}\rangle
    - S
  \right).\label{eq:app-exp3-rt2}
\end{equation}
We decompose $S = S_1 + S_2$,
where
\begin{equation*}
  S_1 = \ex{\sum_{t=1}^{T} \sum_{x \in \aset}{ q_t(x)\langle \phi_t, \wx \rangle}}, \quad
  S_2 = \ex{\sum_{t=1}^{T} \sum_{x \in \aset}{ q_t(x)\langle \theta_t -\phi_t, \wx \rangle}}.
\end{equation*}
First, we bound $|S_2|$. To do this, we prove the following lemma.
\begin{lem}
  \label{lem:appendix-phi-theta}
  For any $x \in \aset$, the following inequality holds:
  \begin{equation*}
    \left|\ex[t-1]{
    \langle \phi_t - \theta_t, \wx \rangle}
    \right|
  \le \frac{\epsilon\Gamma(\piexp)}{\gamma}.
  \end{equation*}
  In particular, we have $
  \left|\ex{ \langle \phi_t - \theta_t, \wx \rangle}\right|
   \le \frac{\epsilon\Gamma(\piexp)}{\gamma} $.
\end{lem}
\begin{proof}
  We note that by conditioning on $\mathcal{H}_{t-1}$, randomness comes only from $x_t$.
  By definition of $\phi_t$, we have
  \begin{align*}
    \ex[t - 1]{\langle \phi_t, \wx \rangle} &= \ex[t-1]{\langle \left(
      \langle \theta_t, \wx_t \rangle + \omega_t(x_t)\right)
      Q_t^{-1} \wx_t, \wx \rangle
     }\\
     &=\ex[t-1]{\wx^{\trn} Q_t^{-1}\wx_t\wx_t^{\trn} \theta_t}
     + \ex[t-1]{\omega_t(x_t)\wx^{\trn}Q_t^{-1}\wx_t}\\
     &= \langle \theta_t, \wx \rangle + \ex[t-1]{\omega_t(x_t)\wx^{\trn}Q_t^{-1}\wx_t}.
  \end{align*}
  Therefore,
  \begin{align*}
    \left|\ex[t-1]{
      \langle \phi_t - \theta_t, \wx \rangle}    \right|
      \le \epsilon \ex[t-1]{\|\wx_t\|_{Q_t^{-1}} \| \wx \|_{Q_t^{-1}}} \le
      \frac{\epsilon}{\gamma} \ex[t-1]{\|\wx_t\|_{Q(\piexp)^{-1}} \| \wx \|_{Q(\piexp)^{-1}}}
      \le \frac{\epsilon \Gamma(\piexp)}{\gamma}.
  \end{align*}
  Here in the second inequality, we use $\gamma Q(\piexp) \le Q_t$ and
  the last inequality follows from the definition of $\Gamma(\piexp)$.
  The second assertion follows from
  \begin{equation*}
  \left|\ex{ \langle \phi_t - \theta_t, \wx \rangle}\right|
  \le \ex{\left|
  \ex[t-1]{ \langle \phi_t - \theta_t, \wx \rangle}
  \right|}
  \le \frac{\epsilon\Gamma(\piexp)}{\gamma}.
  \end{equation*}
\end{proof}
By this lemma, we can bound $S_2$ as follows.
\begin{lem}
  \label{lem:s2-bound}
  The following inequality holds:
  \begin{equation*}
    |S_2| \le \frac{\epsilon T \Gamma(\piexp)}{\gamma}.
  \end{equation*}
\end{lem}
\begin{proof}
  Since $q_t(x)$ is $\mathcal{H}_{t-1}$-measurable for any $x \in \aset$,
   we have
   \begin{align*}
     S_2 =
     \ex{\sum_{t=1}^{T} \sum_{x \in \aset} q_t(x) \ex[t-1]{ \langle \theta_t -\phi_t, \wx \rangle}}.
   \end{align*}
   Therefore, we have
   \begin{align*}
     |S_2| \le
     \ex{\sum_{t=1}^{T} \sum_{x \in \aset} q_t(x) \left|
    \ex[t-1]{ \langle \theta_t -\phi_t, \wx \rangle}
     \right|}
     \le
     \frac{\epsilon T \Gamma(\piexp)}{\gamma}.
   \end{align*}
   Here we used Lemma \ref{lem:appendix-phi-theta} in the last inequality.
\end{proof}

Next, we introduce the following elementary lemma (c.f. \citet[Lemma 49]{chatterji2019online}).
\begin{lem}
  \label{lem:appendix-exp3-elem}
  Let $\eta > 0$ and $X$ be a random variable.
  We assume that $\eta X \le 1$ almost surely.
  Then we have
  \begin{equation*}
    \ex{X}  \ge \frac{1}{\eta} \log \left(\ex{\exp(\eta X)} \right) - ( e- 2) \eta \ex{X^2}.
  \end{equation*}
\end{lem}
\begin{proof}
  By $\log(x) \le x - 1$ for $x > 0$ and
  $\exp(y) \le 1 + y + (e-2)y^2$ for $y \le 1$,
  we have
  \begin{align*}
    \log \ex{\exp(\eta X)} \le \ex{\exp(\eta X) } - 1
    \le \ex{\eta X + (e-2)\eta^2 X^2}.
  \end{align*}
\end{proof}
To apply the lemma with $X = \langle \phi_t, \wx \rangle$ and $\mathbf{E} = \mathbf{E}_{x \sim q_t}$,
we prove the following:
\begin{lem}
  \label{lem:appendix-gamma}
  Let $x \in \aset$ and assume that $\gamma = \eta B \Gamma(\piexp)$.
  Then, we have $\eta |\langle \phi_t, \wx \rangle| \le 1$.
\end{lem}
\begin{proof}
  By definition of $\phi_t$, we have
  \begin{align*}
    \eta  |g_t(x_t)\wx_t^{\trn} Q_t^{-1} \wx|
    \le \eta B \| \wx_t\|_{Q_t^{-1}} \|\wx \|_{Q_t^{-1}}
    \le \frac{\eta B \Gamma(\piexp)}{\gamma}.
  \end{align*}
  Here, in the last inequality, we use
  $\gamma Q(\piexp) \le Q_t$ and the definition of $\Gamma(\piexp)$.
\end{proof}

By Lemma \ref{lem:appendix-exp3-elem} and Lemma \ref{lem:appendix-gamma}, we obtain the following.
\begin{equation}
  \label{eq:s1_lb}
  S_1  \ge \frac{U_1}{\eta} - (e-2)\eta U_2,
\end{equation}
where $U_1$ and $U_2$ are given as
\begin{align*}
  U_1 = \ex{\sum_{t=1}^{T}\log\left(
    \sum_{x \in \aset}{q_t(x)\exp\left(\eta \langle \phi_t, \wx \rangle \right)}
  \right)}, \quad
  U_2 = \ex{\sum_{t=1}^{T} \sum_{x \in \aset}{q_t(x)\langle \phi_t, \wx\rangle^2}}.
\end{align*}
We bound $|U_2|$ as follows.
\begin{lem}
  \label{lem:appendix-u2-bound}
  The following inequality holds:
  \begin{equation*}
    |U_2| \le \frac{B^2 DT}{1 - \gamma}.
  \end{equation*}
\end{lem}
\begin{proof}
  By definition of $\phi_t$, we have
  \begin{align*}
    &\sum_{x \in \aset} q_t(x) \langle \phi_t, \wx \rangle^2
    \le B^2 \sum_{x \in \aset} q_t(x) \wx_t^{\trn} Q_t^{-1} \wx \wx^{\trn}
    Q_t^{-1}\wx_t
    = B^2 \wx_t Q_t^{-1} Q(q_t) Q_t^{-1}\wx_t\\
    &\le \frac{B^2}{1-\gamma} \wx_t^{\trn} Q_t^{-1}\wx_t.
  \end{align*}
  Here the last inequality follows from $(1 - \gamma)Q(q_t) \le Q_t$.
  Therefore,
  \begin{align*}
    &\ex{\sum_{x\in \aset} \langle \phi_t, \wx \rangle^2}
    \le \frac{B^2}{1-\gamma}\ex{\wx_t^{\trn}Q_t^{-1}\wx_t}
     = \frac{B^2}{1-\gamma}\ex{\Tr \left(\wx_t \wx_t^{\trn}Q_t^{-1}\right)}\\
    & = \frac{B^2}{1- \gamma}\ex{\Tr \left(Q_t Q_t^{-1}\right)} = \frac{B^2D}{1 - \gamma}.
  \end{align*}
  Here the second equality follows from the fact that $Q_t$ is $\mathcal{H}_{t-1}$-measurable and
  the linearity of the trace.
  The assertion of the lemma follows from this.
\end{proof}
Next, we give a lower bound for $U_1$.
\begin{lem}
  \label{lem:appendix-u1-lb}
  Let $x_0 \in \aset$ be any element.
  Then the following inequality holds:
  \begin{equation*}
    U_1 \ge \eta \ex{\sum_{t=1}^{T} \langle \phi_t, \wx_0 \rangle} - \log(|\aset|).
  \end{equation*}
\end{lem}
\begin{proof}
  By definition of $q_t$, we have
  \begin{align*}
    U_1 &=
    \ex{\sum_{t=1}^{T} \left\{
      \log \left( \sum_{x \in \aset} \exp\left(
      \eta \sum_{s=1}^{t} \langle \phi_s, \wx \rangle
    \right)
    \right)
    -
    \log\left(
      \sum_{x \in \aset}\exp\left(\eta \sum_{s=1}^{t-1} \langle \phi_s, \wx \rangle\right)
    \right)\right\}
    }\\
    &=\ex{
      \log \left( \sum_{x \in \aset} \exp\left(
      \eta \sum_{s=1}^{T} \langle \phi_s, \wx \rangle
    \right)
    \right)
    } - \log|\aset|\\
    &= \eta \ex{\sum_{t=1}^{T} \langle \phi_t, \wx_0\rangle} - \log|\aset|.
  \end{align*}
\end{proof}

\begin{proof}[Proof of Proposition \ref{prop:appexp3}]
  We assume $\gamma = \eta B \Gamma(\piexp)$.
  By \eqref{eq:app-exp3-rt2},
  we have
  \begin{align*}
    \ex{R(T)}
    \le 2 \epsilon T +
    2\gamma BT + (1-\gamma)\ex{\sum_{t=1}^{T} \langle \theta_t, \wx_{*}\rangle}
    - (1- \gamma)S_1 + (1-\gamma)|S_2|
  \end{align*}
  By inequality \eqref{eq:s1_lb}, Lemma \ref{lem:s2-bound}, and Lemma \ref{lem:appendix-u1-lb} with $x_0 =
  x_{*}$, we have
  \begin{align*}
    &\ex{R(T)} \le 2 \epsilon T
    + 2\gamma BT +
    (1 -\gamma) \ex{\sum_{t=1}^{T} \langle \theta_t - \phi_t , \wx_{*} \rangle}
    + \frac{1-\gamma}{\eta}\log|\aset| + (e-2)\eta (1- \gamma) U_2 +
    \frac{\epsilon T \Gamma(\piexp)}{\gamma}
    \\
    &\le
    2\epsilon T +
    2 \gamma BT +
    \frac{\epsilon T \Gamma(\piexp)(1-\gamma)}{\gamma} +
    \frac{\log |\aset|}{\eta} + (e-2)\eta B^2 DT +
  \frac{\epsilon T \Gamma(\piexp)}{\gamma}\\
  &\le
   2\epsilon T+
  2 \gamma BT +
  \frac{\log |\aset|}{\eta} + (e-2)\eta B^2 DT +
  \frac{2\epsilon T \Gamma(\piexp)}{\gamma}.
  \end{align*}
  Here in the second inequality, we used Lemma \ref{lem:appendix-phi-theta} and Lemma \ref{lem:appendix-u2-bound}.
  By $\gamma = \eta B \Gamma(\piexp)$ and $\Gamma(\piexp) \le D$,
  we have
  \begin{align*}
    &\ex{R(T)} \le 2\epsilon T + 2 B^2 \eta T \Gamma(\piexp) +
  \frac{2\epsilon T}{B\eta} + \frac{\log|\aset|}{\eta} + (e-2)B^2 \eta DT\\
  &\le
    2 \epsilon T
   + e B^2 \eta DT + \frac{2\epsilon T}{B\eta} + \frac{\log|\aset|}{\eta}.
  \end{align*}
\end{proof}

\subsection{Proof of Proposition \ref{prop:app-ts}}
\label{sec:proof-prop-ts}
We assume $\lambda \ge 1$.
This can be proved by modifying the proof of \citep{agrawal2013thompson}.
Since most of their arguments can be directly applicable to our case,
we omit proofs of some lemmas.
Let $(\Psi, \prob, \mathcal{G})$ be the probability space on which
all random variables considered here are defined,
where $\mathcal{G} \subset 2^{\Psi}$ is a $\sigma$-algebra on $\Psi$.
We put $x^* := \argmax_{x \in \aset}g(x)$
and for $t = 1, 2, \dots, T$ and $x \in \aset$,
we put $\Delta(x) := \langle \wx^*, \theta \rangle - \langle \wx, \theta \rangle$.
We also put $v_t := l_t : = \beta(A_{t-1}, \delta/(2T), \lambda) + \err \psi_{t-1}$
and $g_t := \sqrt{4\log(|\aset|t)} v_t + l_t$.
In each round $t$,
$\mu_{t-1}$ is sampled
from the multinomial normal distribution $\normdist\left(
\theta_{t-1}, \left(\beta(A_{t-1}, \delta/2, \lambda) + \err \psi_t\right)^2A_{t-1}^{-1}
\right)$.
For $t =1, \dots, T$, we define $E_t$ by
\begin{equation*}
  E_t := \left\{
    \psi \in \Psi :
  |\langle \wx, \theta_{t-1} \rangle- \langle \wx, \theta \rangle| \le l_t \| \wx\|_{A_{t-1}^{-1}},
  \quad \forall x \in \aset
  \right\},
\end{equation*}
and define event $E'_t$ by
\begin{equation*}
  E_t' := \left\{
    \psi \in \Psi:
  | \langle \wx, \mu_{t-1} \rangle - \langle \wx, \theta_{t-1} \rangle| \le
  \sqrt{4\log(|\aset|t)} v_t \|\wx\|_{A_{t-1}^{-1}}, \quad \forall x \in \aset
  \right\}.
\end{equation*}
For an event $G$, we denote by $1_G$ the corresponding indicator function.
Then by assumptions, we see that $E_t \in \filt_{t-1}$, i.e., $1_{E_t}$ is $\filt_{t-1}$-measurable.
For a random variable $X$ on $\Psi$, we say
``on event $E_t$, the conditional expectation (or conditional probability) $\ex{X \mid \filt_{t-1}}$
satisfies a property''
if and only if $1_{E_t}\ex{X \mid \filt_{t-1}} = \ex{1_{E_t}X \mid \filt_{t-1}}$
satisfies the property for almost all $\psi \in \Psi$.

Then by Lemma \ref{lem:upper-bd} and the proof of \citep[Lemma 1]{agrawal2013thompson},
we have
\begin{lem}
  \label{lem:ts-lem1}
  $\prob(E_t) \ge 1 - \frac{\delta}{2T}$ and
  $\prob(E'_t \mid \mathcal{F}_{t - 1}) \ge 1 - 1/t^2$ for all $t$.
\end{lem}


We note that the proof of \citep[Lemma 2]{agrawal2013thompson} works if $l_t \le v_t$,
i.e., we have the following lemma:
\begin{lem}
  \label{lem:ts-lem2}
  On event $E_t$, we have
  \begin{equation*}
    \prob\left(\langle \mu_t, \wx^*\rangle > \langle \theta, \wx^*\rangle
    \mid \filt_{t - 1}\right) \ge p,
  \end{equation*}
  where $p = \frac{1}{4e\sqrt{\pi}}$.
\end{lem}

The main differences of our proof and theirs lie in the definitions of
$l_t, v_t, x^*$, and $\Delta(x)$ (they define $\Delta(x)$ as
$\sup_{y \in\aset} \langle \theta, \widetilde{y} \rangle - \langle \theta, \wx\rangle$
and we consider $x^* = \argmax g(x)$ instead of $\argmax_{x} \langle \wx, \theta \rangle$).
However, it can be verified that these differences do not matter in the arguments of
Lemma 3, 4 in \citep{agrawal2013thompson}.
In fact,
we can prove the following lemma in a similar way to the proof of \citep[Lemma 3]{agrawal2013thompson}.
\begin{lem}
  We define $C(t)$ by $\{x \in \aset : \Delta(x) > g_t \| \wx\|_{A_{t-1}^{-1}}\}$.
  On event $E_t$, we have
  \begin{equation*}
    \prob (x \not \in C(t) \mid \filt_{t-1}) \ge p  - \frac{1}{t^2}.
  \end{equation*}
  \label{lem:ts-lem3}
  Here $p$ is given in Lemma \ref{lem:ts-lem2}.
\end{lem}
\begin{proof}
  Because the algorithm selects $x \in \aset$ that maximizes
  $\langle \wx, \mu_{t-1} \rangle$, if
  $\langle \wx^*, \mu_{t-1} \rangle > \langle \wx, \mu_{t-1}\rangle$
  for all $x \in C(t)$, then we have
  $x_t \not \in C(t)$.
  Therefore, we have
  \begin{equation}
    \label{ineq-lem:ts-lem3}
  \prob( x_t \not \in C(t) \mid \filt_{t-1}) \ge
  \prob \left(\langle \wx^*, \mu_{t-1} \rangle > \langle \wx, \mu_{t-1}\rangle,
  \forall x \in C(t) \mid \filt_{t-1} \right).
  \end{equation}
  By definitions of $C(t), E_t$, and $E_t'$,
  on even $E_t \cap E_t'$, we have
  $\langle \wx, \mu_{t-1}\rangle \le \langle \wx, \theta \rangle + g_t \| \wx \|_{A_{t-1}^{-1}}
  < \langle \wx^*, \theta \rangle$
  for all $x \in C(t)$.
  Therefore, on $E_t \cap E_t'$,
  if $\langle \wx^*, \mu_{t-1}\rangle > \langle \wx^*, \theta \rangle$,
  we have $\langle \wx^*, \mu_{t-1}\rangle > \langle \wx, \mu_{t-1}\rangle$ for all $x \in C(t)$.
  Thus we obtain the following inequalities:
  \begin{align*}
    &\prob(\langle \wx^*, \mu_{t-1}\rangle > \langle \wx, \mu_{t-1}\rangle, \quad \forall x \in C(t) \mid \filt_{t-1})\\
    &\ge \prob( \langle \wx^*, \mu_{t-1}\rangle > \langle \wx^*, \theta \rangle \mid \filt_{t-1})
    - \prob((E_t')^c \mid \filt_{t-1})\\
    &\ge p - 1/t^2.
  \end{align*}
  Here $(E_t')^c$ is the complement of $E_t'$ and we used Lemmas \ref{lem:ts-lem1}, \ref{lem:ts-lem2}
  in the last inequality.
  By inequality \eqref{ineq-lem:ts-lem3}, we have our assertion.
\end{proof}

We can also prove the following lemma in a similar way to the proof of \citep[Lemma 4]{agrawal2013thompson}.
\begin{lem}
  \label{lem:lem4}
  On event $E_t$, we have
  \begin{equation*}
    \ex{\Delta( x_t ) \mid \mathcal{F}_{t-1}} \le c_1 g_t \ex{\|x_t\|_{A_{t-1}^{-1}}
      \mid \mathcal{F}_{t-1}}
    + \frac{c_2 g_t}{t^2},
  \end{equation*}
  where $c_1$ and $c_2$ are universal constants.
\end{lem}

For $t=1, 2, \dots, T$, define random variables $X_t$ and $Y_t$ by
\begin{equation*}
  X_t := \Delta(x_t) 1_{E_t} - c_1 g_t \|x_t\|_{A_{t-1}^{-1}} - \frac{c_2 g_t}{t^2},\quad
  Y_t := \sum_{s=1}^{t}X_t.
\end{equation*}
From Lemma \ref{lem:lem4}, we can prove the following lemma.
\begin{lem}
  \label{lem:super-mart}
  The process $\{Y_t\}_{t=0, \dots, T}$ is a super-martingale process w.r.t.
  the filtration $\{\mathcal{F}_t\}_t$.
\end{lem}

\begin{proof}[Proof of Proposition \ref{prop:app-ts}]
  By Lemma \ref{lem:super-mart} and $\| X_t \| \le 2(B + \err) + (c_1 + c_2)g_t$ (for all $t$),
  applying Azuma-Hoeffding inequality,
  we see that
  there exists an event $G$ with $\prob(G) \ge 1 - \delta/2$
  such that on $G$, the following inequality holds:
  \begin{equation*}
    \sum_{t=1}^T \Delta(x_t) 1_{E_t}
    \le \sum_{t=1}^{T} c_1 g_t \| \wx_t\|_{A_{t-1}^{-1}}+
    \sum_{t=1}^{T}c_2 g_t/t^2 + \sqrt{
      \left(
      4T(B+\err)^2 + 2(c_1 + c_2)^2 \sum_{t=1}^{T}g_t^2
      \right)\log(2/\delta)
    }.
  \end{equation*}
  Since $g_t \le g_T$ for any $t$, on the event $G$, we have
  \begin{align*}
    \sum_{t=1}^T \Delta(x_t) 1_{E_t}  \le c_1 g_T
    \sqrt{T}\sqrt{\sum_{t=1}^{T} \| \wx_t\|^2_{A_{t-1}^{-1}}} +
    c_2 g_T \frac{\pi^2}{6} +
    \sqrt{T}\sqrt{\left(
      4(B + \err)^2 + 2(c_1 + c_2)^2 g_T^2
      \right)\log(2/\delta)}.
  \end{align*}
  By inequalities \eqref{eq:beta-upp}, we have
  \begin{align*}
    \sqrt{\sum_{t=1}^{T} \| \wx_t\|^2_{A_{t-1}^{-1}}} &= \otilde(\sqrt{D}),\quad
    g_T = \otilde(\sqrt{\log (|\aset|)}v_T) = \otilde
    \left(\sqrt{|\log |\aset|} (\sqrt{D} + \sqrt{\log(1/\delta)} + \err\psi_T) \right).
  \end{align*}
  Since $\psi_T = \sum_{s=1}^{T}\| \wx_s\|_{A_{s-1}^{-1}} \le
  \sqrt{T} \sqrt{\sum_{s=1}^{T}\|\wx_s\|^2_{A_{s-1}^{-1}}} = \otilde (\sqrt{DT})$,
  we obtain
  \begin{equation*}
    g_T = \otilde\left(
      \sqrt{\log A}
      (\sqrt{D} + \sqrt{\log(1/\delta)} + \err \sqrt{DT})
      \right).
  \end{equation*}
  Therefore, on the event $G$, we have
  \begin{equation*}
    \sum_{t=1}^T \Delta(x_t) 1_{E_t}
    =
    \otilde \left(
      \sqrt{\log (|\aset|)}
    \left\{
      D\sqrt{T} + \sqrt{DT\log(1/\delta)} + \log(1/\delta) \sqrt{T} +
     \left(D T + T \sqrt{D\log(1/\delta)}\right)\err)
    \right\}
    \right).
  \end{equation*}
  Therefore, on event $\bigcap_{t=1}^{T}E_t \cap G$, we can upper bound the regret as follows:
  \begin{align*}
    R(T)  &= \sum_{t=1}^{T} \{g(x^*) - g(x_t)\}
    \le \err T +  \sum_{t=1}^T \Delta(x_t) 1_{E_t}\\
    &=
    \otilde \left(
      \sqrt{\log (|\aset|)}
    \left\{
      D\sqrt{T} + \sqrt{DT\log(1/\delta)} + \log(1/\delta) \sqrt{T} +
     \left(D T + T \sqrt{D\log(1/\delta)}\right)\err)
    \right\}
    \right).
  \end{align*}
  Since $\prob(\bigcap_{t=1}^{T}E_t\cap G)\ge 1 - \delta$, we have the assertion of the proposition.
\end{proof}

\subsection{Proof of Theorem \ref{thm:app-ucb}}
\label{sec:subsec-proof-thm-app-ucb}
Since Theorems \ref{thm:app-pe}, \ref{thm:app-ts} can be proved in a similar way,
we only provide proof of Theorem \ref{thm:app-ucb}.

Let $\{\xi_1, \dots, \xi_D\}$ and $N_1, \dots, N_D$
be a sequence of points and Newton basis returned by Algorithm \ref{alg:appucb} with
$\admerr = \frac{\alpha}{T^q}$,
where $D = D_{q, \alpha}(T)$ and $q \ge 1/2$.

We verify the assumptions of the (stochastic) misspecified linear bandit problem hold, i.e.,
we show there exists $\theta \in \RR^D$
such that the following conditions are satisfied for $\wx = [N_1(x),\dots, N_D(x)]^\trn$ and $\theta$:
\begin{enumerate}
\item $\| \wx\|_2 \le 1$.
\item If $x$ is a $\aset$-valued random variable
  and $\mathcal{F}_{t}$-measurable, then $\wx$ is $\mathcal{F}_{t}$-measurable .
\item $\| \theta\|_2 \le B$.
\item $\sup_{x \in \aset}|f(x) - \langle \theta, \wx \rangle| < \err$, where $\err = \alpha B/T^q$.
\end{enumerate}

We put $X_D := \{\xi_1, \dots, \xi_D\}$. Then by definition, Newton basis $N_1, \dots, N_D$
is a basis of $V(X_D)$.
We define $\theta_1, \dots, \theta_D \in \RR$ by
$\Pi_{V(X_D)}f = \sum_{i=1}^{D} \theta_i N_i$ and put $\theta = [\theta_1, \dots, \theta_D]^\trn$.
Since Newton basis is an orthonormal basis of $V(X_D)$,
we have
\begin{equation*}
  \| \theta \|_2 = \left \| \sum_{i=1}^{D} \theta_i N_i\right \|_{\hilb{\Omega}}
  = \left \|\Pi_{V(X_D)}f \right \|_{\hilb{ \Omega }}
  \le \| f \|_{\hilb{\Omega}} \le B.
\end{equation*}
By the orthonormality, we have
\begin{math}
  P_{V(X)}^2(x) = K(x, x) - \sum_{i = 1}^{m} N_i^2(x)
\end{math}
(c.f. \citet[Lemma 5]{santin2017convergence}).
Then by assumption, we have $\| \wx\|_2^2 = \sum_{i=1}^{m}N_i^2(x) = K(x, x) - P_{V(X_D)}^2(x) \le 1$.
Since $N_k$ for $k=1, \dots, D$ is a linear combination of $K(\cdot, \xi_1), \dots, K(\cdot, \xi_D)$
and $K$ is continuous, $x \mapsto \wx$ is continuous.
Therefore, $\wx$ is $\filt_t$-measurable if $x$ is $\filt_t$-measurable.
By definition of the P-greedy algorithm, we have $\sup_{x \in \aset}P_{V(X_D)}(x)
< \frac{\alpha}{T^q}$.
By this inequality and the definition of the power function, the following inequality holds:
\begin{equation*}
  \sup_{x \in \aset} |f(x) - \langle \theta, \wx \rangle|
  =\sup_{x \in \aset} |f(x) - \left(\Pi_{V(X_D)}f\right)(x)| \le \| f\| \frac{\alpha}{T^q} \le
  \frac{\alpha B}{T^q}.
\end{equation*}
Thus, one can apply results of a misspecified linear bandit problem with $\err = \frac{\alpha B}{T^q}$.
By applying Proposition \ref{prop:app-ucb},
with probability at least $1-\delta$, the regret is upper bounded as follows:
\begin{equation*}
  R_{\algucb}(T)  =
   \otilde\left(
    \sqrt{T D_{q, \alpha}(T) \log(1/\delta)} + D_{q, \alpha}(T) \sqrt{T}
    \right).
\end{equation*}
Since computing Newton basis requires $O(|\aset| D^2)$ time and
total complexity of the modified LinUCB is given as
$O(|\aset| D^2 T)$, we have the assertion of Theorem \ref{thm:app-ucb}.

\subsection{Proof of Theorem \ref{thm:app-exp3}}
For simplicity, by normalization, we assume $B = 1$.
We denote by $R_{\text{\algexp{}}}(T)$
the cumulative regret that \algexp{} with $q = 1$
and $\alpha = \log (|\aset|)$ incurs up to time step $T$.
We can reduce the adversarial RKHS bandit problem to the adversarial
misspecified linear bandit problem as in \S \ref{sec:subsec-proof-thm-app-ucb}.
To apply Proposition \ref{prop:appexp3},
we need to prove that $\{\wx | x \in \aset \}$ spans $\RR^D$.
We denote by $X = \{X_1, \dots, X_D\}$ the points returned by the $P$-greedy algorithm.
Then,
since $N_1, \dots, N_D$ is a basis of $V(X)$ and $K$ is positive definite,
$\rank (N_i(x))_{1 \le i \le D, x \in \aset} = \rank (K(x_i, x))_{1 \le i \le D, x \in \aset} = D$.
Therefore, $\{\wx | x \in \aset \}$ spans $\RR^D$.

By Proposition \ref{prop:appexp3}, we have
\begin{equation*}
  \ex{R_{\text{\algexp{}}}(T)} \le 2 \epsilon T +e \eta D T + \frac{2\epsilon} {\eta} +
  \frac{\log(|\aset|)}{\eta},
\end{equation*}
where $\epsilon = \frac{\log (|\aset|)}{T}$ and $D = D_{1, \log (|\aset|)}(T)$.
Thus we have
\begin{math}
  \ex{R_{\text{\algexp{}}}(T)} \le 2 \log(|\aset|) +e \eta D T + \frac{3\log(|\aset|)}{\eta}.
\end{math}
By taking $\eta = \sqrt{\frac{\log (|\aset|)}{DT}}$, we have the assertion of the theorem.
\subsection{Proof of Theorem \ref{thm:igp-app}}
\label{sec:appendix-thm-igp-app}
First, we prove that the $P$-greedy algorithm (Algorithm \ref{alg:nbasispgreedy})
also gives a uniform kernel approximation.
\begin{lem}
  \label{lem:kernel-app}
  Let $N_1, \dots, N_D$ be a Newton basis returned by the $P$-greedy algorithm \ref{alg:nbasispgreedy}
  with error $\admerr$ and $\omdsc = \aset$.
  For $x \in \aset$, we put $\wx := [N_1(x), \dots, N_D(x)]^{\trn}$.
  Then, we have
  \begin{math}
    \sup_{x, y \in \aset}|K(x, y) - \langle \wx, \mywy \rangle| \le \admerr.
  \end{math}
  \begin{proof}
    We denote by $X$ the points returned by the $P$-greedy algorithm.
    Then, by definition of the Power function, we have
    \begin{equation*}
      \left|h(x) - \left(\Pi_{V(X)}h \right)(x)\right| \le \| h\|_{\hilb{\Omega}} \admerr,
    \end{equation*}
    for any $h \in \hilb{\Omega}$ and $x \in \aset$.
    We take arbitrary $y \in \aset$ and take $h = K(\cdot, y)$.
    Since $N_1, \dots, N_D$ is an orthonormal basis of $V(X)$,
    we have
    \begin{align*}
      \left(\Pi_{V(X)} h\right)(x)
      = \sum_{i=1}^{D}\langle h, N_i\rangle_{\hilb{\Omega}}N_i(x)
      = \sum_{i=1}^{D}N_i(y)N_i(x) = \langle \wx, \mywy \rangle.
    \end{align*}
    Here, in the second equality, we used the reproducing property.
    Since $\| h\|_{\hilb{\Omega}} \le 1$ and $x, y$ are arbitrary,
    we have our assertion.
  \end{proof}
\end{lem}
Next, we introduce the following classical result on matrix eigenvalues.
\begin{lem}[a specail case of the Wielandt-Hoffman theorem \citet{hoffman1953variation}]
  \label{lem:wielandt-hoffman}
  Let $A, B \in \RR^{n\times n}$ be symmetric matrices.
  Denote by $a_1 \le \dots \le a_n$ and $b_1 \le  \dots \le b_n$ be the eigenvalues of
  $A$ and $B$ respectively.
  Then, we have
  \begin{math}
    \sum_{i=1}^{n}|a_i - b_i|^2 \le \| A- B \|_{F}^{2},
  \end{math}
  where $\| \cdot \|_{F}$ denotes the Frobenius norm.
\end{lem}
By these lemmas, we can prove $\log \det\left( \lambda^{-1} A_T \right)$
is an approximation of the maximum information gain.
\begin{lem}
  \label{lem:at_gamma-app}
  We apply \algucbh{} with admissible error $\admerr$
  to the stochastic RKHS bandit, then following inequality holds:
  \begin{equation*}
    \log \det
    \left(
      \lambda^{-1} A_T
    \right)
    \le 2\gamma_T + \frac{\admerr T^{3/2}}{\lambda}.
  \end{equation*}
\end{lem}
\begin{proof}
  We define a $T \times T$ matrix $\widetilde{K}_T$
  as $(\langle \wx_i, \wx_j \rangle)_{1 \le i, j \le T}$.
  Since for any matrix $X \in \RR^{n \times m}$,
  $\det (1_n + X X^{\trn}) = \det (1_m + X^{\trn} X)$ holds,
  we have $\det(\lambda^{-1}A_{T}) = \det\left(1_T + \lambda^{-1} \widetilde{K}_T
  \right)$.
  We denote by $\rho_1 \le \dots \le \rho_T$ the eigenvalues of $K_T$
  and $\wrho_1 \le \dots \le \wrho_T$ those of $\widetilde{K}_T$.
  Then by the Wielandt-Hoffman theorem (Lemma \ref{lem:wielandt-hoffman}),
  we have
  \begin{equation}
    \label{eq:wrho_rho}
  \sqrt{\sum_{i=1}^{T} (\rho_i - \wrho_i)^2} \le \lambda^{-1}
  \|K_T - \widetilde{K}_T\|_{F}
  \le \lambda^{-1} \admerr T,
  \end{equation}
  where the last inequality follows from Lemma \ref{lem:kernel-app}.
  Thus, we have
  \begin{align*}
    \log \det \left( \lambda^{-1} A_T \right)
    &= \log \det \left (1_T + \lambda^{-1}\widetilde{K}_T \right)
    = \sum_{i=1}^{T}\log(\wrho_i)
    = \sum_{i=1}^{T} \log(\rho_i) + \sum_{i=1}^{T}\log(\wrho_i / \rho_i)\\
    &\le \log \det (1_T + \lambda^{-1}K_T) + \sum_{i=1}^{T}\frac{\wrho_i - \rho_i}{\rho_i}\\
    &\le \log \det (1_T + \lambda^{-1}K_T) + \sum_{i=1}^{T} |\wrho_i - \rho_i|\\
    &\le \log \det (1_T + \lambda^{-1}K_T) + \frac{\admerr T^{3/2}}{\lambda}.
  \end{align*}
  Here in the second inequality,
  we used $\rho_i \ge 1$
  and in the third inequality,
  we used inequality \eqref{eq:wrho_rho} and the Cauchy-Schwartz inequality.
  Noting that $\log \det (1_T + \lambda^{-1}K_T) \le 2 \gamma_T$ \citep{chowdhury2017kernelized},
  we have our assertion.
\end{proof}

We provide a more precise result than Theorem \ref{thm:igp-app}.
We can prove the following by Proposition \ref{prop:app-ucb}.
\begin{prop}
  We assume that
  $\lambda^{-1}\log \left(\det( \lambda^{-1}A_T)\right)
   \le 2 \gamma_T + \delta_T$,
  where $\delta_T = O(T^{a - q})$ with $a \in \RR$
  and $q$ is the parameter of \algucbh{}.
  We also assume that $\delta_T = O(\gamma_T)$ and $\lambda = 1$.
  Then with probability at least $1 - \delta$,
  the cumulative regret of \algucbh{} is upper bounded by a function $b(T)$,
  where $b(T)$ is given as
  \begin{equation}
    b(T) = 4 \betaigp_T \sqrt{\gamma_T T} + O(\sqrt{T\gamma_T} T^{(a-q)/2}
    + \gamma_T T^{1-q}),
    \label{eq:bt-upper-bound}
  \end{equation}
  where $\betaigp_T$ is defined by $B + R\sqrt{2(\gamma_T + 1 + \log(1/\delta))}$.
\end{prop}
\begin{rem}
  We note that the cumulative regret of IGP-UCB is upper bounded by
  $4 \betaigp_T \sqrt{\gamma_T (T + 2)}$ by the proof in \citep{chowdhury2017kernelized}.
\end{rem}
If $q > \max(a, 1/2)$, then the first term $4 \betaigp_T \sqrt{\gamma_T T}$ in
\eqref{eq:bt-upper-bound} is the main term of $b(T)$.
By Lemma \ref{lem:at_gamma-app}, we can take $a = 3/2$.
Thus, we have the assertion of Theorem \ref{thm:igp-app}.

\section{Supplement to the Experiments}
\label{sec:appendix-exp}
\subsection{Experimental Setting}
For each reward function $f$, we add independent Gaussian noise of mean $0$
and standard deviation $0.2 \cdot \| f\|_{L^1(\aset)}$.
We use the $L^1$-norm because even if we normalize $f$
so that $\| f\|_{\hilb{\Omega}} = 1$, the values of the function $f$ can be small.
As for the parameters of the kernels,
we take $\mu = 2d$ for the RQ kernel because
the condition $\mu = \Omega(d)$ is required for positive definiteness.
We take $l = 0.3\sqrt{d}$ and $l = 0.2\sqrt{d}$
if the kernel is RQ kernel and SE kernel respectively
because the diameter of the $d$-dimensional cube is $\sqrt{d}$.
As for the parameters of the algorithms,
we take $B = 1, \delta=10^{-3}$ and $R = 0.2 \cdot \left(\sum_{i=1}^{10}\|f_i \|_{L^1(\aset)}/10\right)$
for both algorithms,
where $f_1, \dots, f_{10}$ are the reward functions used for the experiment.
We take $\lambda = 1, \alpha = 5 \cdot 10^{-3}, q = 1/2$ for \algucbh{}
and $\lambda = 1 + 2/T$ for IGP-UCB.

Since exact value of the maximum information gain is not known,
when computing UCB for IGP-UCB, we modify IGP-UCB as follows.
Using notation of \citep{chowdhury2017kernelized},
IGP-UCB selects an arm $x$ maximizing $\mu_{t-1}(x) + \beta_t \sigma_{t-1}(x)$,
where $\beta_t = B + R\sqrt{2(\gamma_{t-1} + 1 + \log(1/\delta))}$.
Since exact value of $\gamma_{t-1}$ is not known,
we use $\frac{1}{2} \ln \det(I + \lambda^{-1} K_{t-1})$ instead of $\gamma_{t-1}$.
From their proof, it is easy to see that this modification of IGP-UCB have
the same guarantee for the regret upper bound as that of IGP-UCB.
In addition,
by
$\ln \det(I + \lambda^{-1} K_t) = \sum_{s=1}^{t} \log(1 + \lambda^{-1} \sigma_{s-1}^2(x_s))$,
one can update $\ln \det(I + \lambda^{-1} K_t) $ in $O(t^2)$ time at each round
if $K_t^{-1}$ is known.
To compute the inverse of the regularized kernel matrix $K_t^{-1}$,
we used a Schur complement of the matrix.

Computation was done by Intel Xeon E5-2630 v4 processor with 128 GB RAM.
We computed UCB for each arm in parallel for both algorithms.
For matrix-vector multiplication, we used efficient implementation of the dot product
provided in
\url{https://github.com/dimforge/nalgebra/blob/dev/src/base/blas.rs}.

\subsection{Additional Experimental Results}
As shown in the main article and \S \ref{sec:appendix-thm-igp-app}, the error $\admerr$ balances
the computational complexity and cumulative regret, i.e.,
if $\admerr$ is smaller, then the cumulative regret is smaller, but
the computational complexity becomes larger.
In this subsection, we provide additional experimental results
by changing $\alpha$ with fixed $q = 1/2$.
We also show results for more complicated reward functions,
i.e. $l = 0.2 \sqrt{d}$ for RQ kernels ($\mu$ is the same)
and $l = 0.1\sqrt{d}$ for SE kernels.

In Table \ref{tab:pd-dims}, we show
the number of points returned by the $P$-greedy algorithms
for the RQ and SE kernels.
\begin{table}[htbp]
  \centering
  \caption{The Number of Points Returned by the P-greedy Algorithm with $\epsilon = \frac{5 \cdot 10^{-3}}{\sqrt{T}}.$
  }
  \begin{tabular}{lllll}
   {}  & RQ ($l = 0.3\sqrt{d}$) & SE ($l = 0.2\sqrt{d}$)
   & RQ ($l = 0.2\sqrt{d}$) & SE ($l = 0.1\sqrt{d}$)\\
   \hline
    $d = 1$ &  18 & 15 & 23& 25\\
    $d = 2$ & 105 & 108 & 188 & 283\\
    $d = 3$ & 376 & 457 & 725 & 994
  \end{tabular}
  \label{tab:pd-dims}
\end{table}




In Figures \ref{fig:rq_alpha_l_normal}, \ref{fig:se_alpha_l_normal}
and Tables \ref{tab:time-rq_alpha_l_normal}, \ref{tab:time-se_alpha_l_normal},
we show the dependence on the parameter $\alpha$.
In these figures,
we denote \algucbh{} with parameter $\alpha$
by \algucbh{}($\alpha$).

In Figures \ref{fig:rq_alpha_l_small}, \ref{fig:se_alpha_l_small}
and Tables \ref{tab:time-rq_alpha_l_small}, \ref{tab:time-se_alpha_l_small},
we also show the dependence on the parameter $\alpha$ for more complicated functions.

\begin{figure}[htbp]
  \centering
  \includegraphics[width=\textwidth]{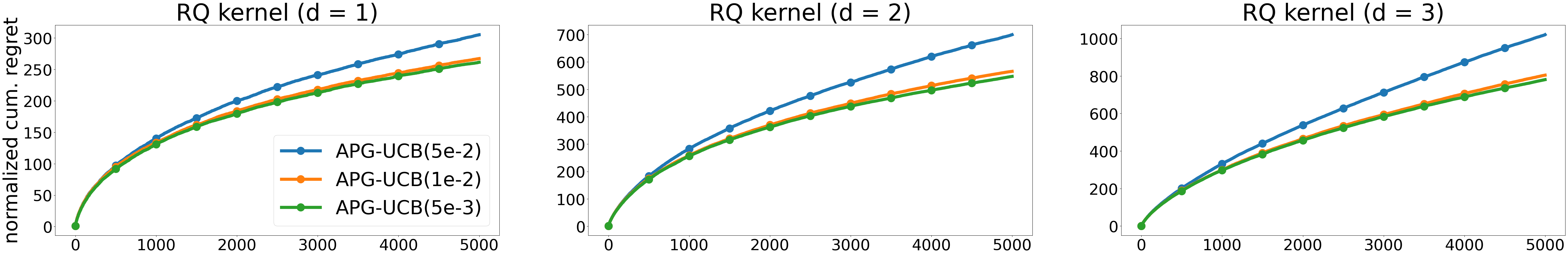}
  \caption{Normalized Cumulative Regret for RQ kernels with $l=0.3\sqrt{d}$.}
  \label{fig:rq_alpha_l_normal}
\end{figure}

\begin{figure}[htbp]
  \centering
  \includegraphics[width=\textwidth]{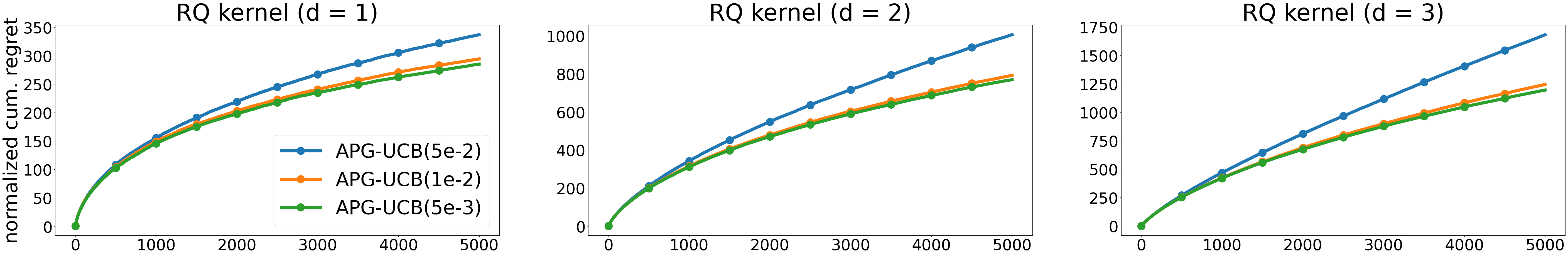}
  \caption{Normalized Cumulative Regret for RQ kernels with $l=0.2\sqrt{d}$.}
  \label{fig:se_alpha_l_normal}
\end{figure}

\begin{figure}[htbp]
  \centering
  \includegraphics[width=\textwidth]{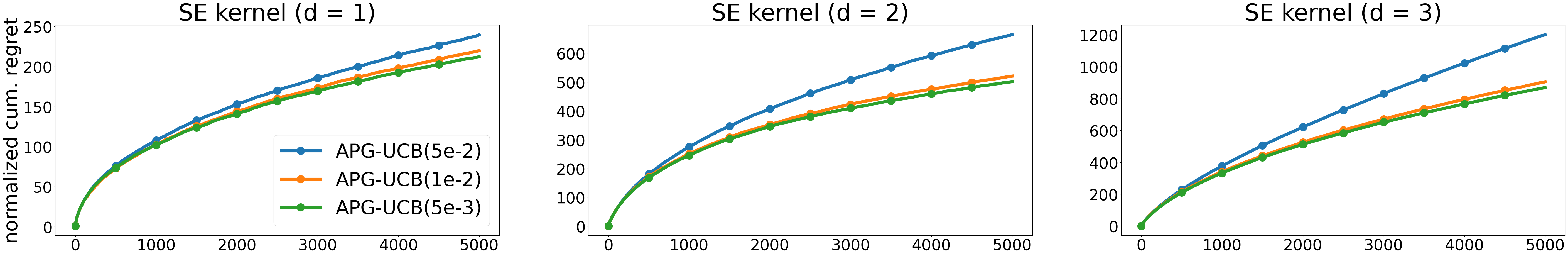}
  \caption{Normalized Cumulative Regret for SE kernels with $l=0.2\sqrt{d}$.}
  \label{fig:rq_alpha_l_small}
\end{figure}

\begin{figure}[htbp]
  \centering
  \includegraphics[width=\textwidth]{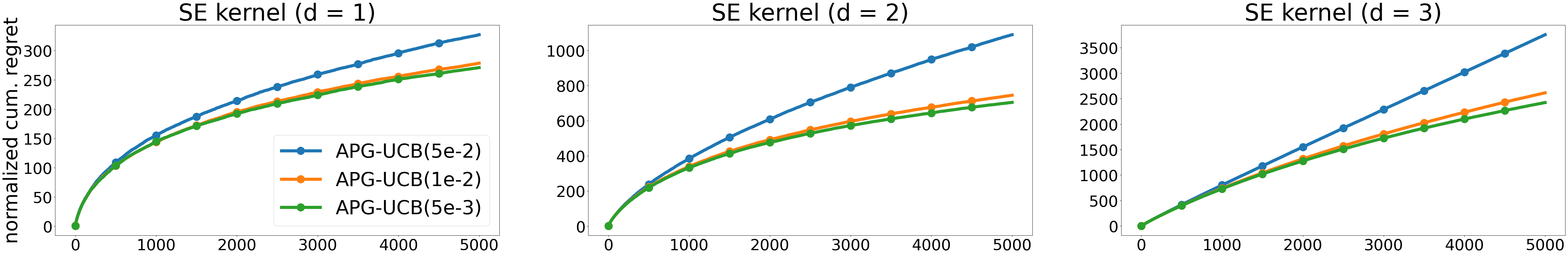}
  \caption{Normalized Cumulative Regret for SE kernels with $l=0.1\sqrt{d}$.}
  \label{fig:se_alpha_l_small}
\end{figure}

\begin{table}[htbp]
  \centering
  \caption{Total Running Time for RQ Kernels with $l=0.3\sqrt{d}$.}
  \begin{tabular}{lllll}
     {} & \algucbh{}(5e-2) & \algucbh{}(1e-2) & \algucbh{}(5e-3) \\
     \hline
     d = 1 (RQ) &     3.91e-01 &     4.06e-01 &     4.23e-01 \\
     d = 2 (RQ) &     1.36e+00 &     2.39e+00 &     2.76e+00 \\
     d = 3 (RQ) &     1.19e+01 &     2.40e+01 &     2.98e+01 \\
  \end{tabular}
  \label{tab:time-rq_alpha_l_normal}
\end{table}

\begin{table}[htbp]
  \centering
  \caption{Total Running Time for SE Kernels with $l=0.2\sqrt{d}$.}
  \begin{tabular}{lllll}
  {} & \algucbh{}(5e-2) & \algucbh{}(1e-2) & \algucbh{}(5e-3) \\
  \hline
  d = 1 (SE) &     3.84e-01 &     4.04e-01 &     4.02e-01 \\
  d = 2 (SE) &     1.69e+00 &     2.59e+00 &     2.89e+00 \\
  d = 3 (SE) &     2.13e+01 &     3.51e+01 &     4.30e+01 \\
  \end{tabular}
  \label{tab:time-se_alpha_l_normal}
\end{table}

\begin{table}[htpb]
  \centering
  \caption{Total Running Time for RQ Kernels with $l=0.2\sqrt{d}$.}
  \begin{tabular}{lllll}
    {} &  \algucbh{}(5e-2) & \algucbh{}(1e-2) & \algucbh{}(5e-3) \\
    \hline
    d = 1 (RQ) &     4.49e-01 &     4.84e-01 &     4.96e-01 \\
    d = 2 (RQ) &     3.84e+00 &     6.01e+00 &     7.39e+00 \\
    d = 3 (RQ) &     4.87e+01 &     8.76e+01 &     1.07e+02 \\
  \end{tabular}
  \label{tab:time-rq_alpha_l_small}
\end{table}

\begin{table}[htbp]
  \centering
  \caption{Total Running Time for SE Kernels with $l=0.1\sqrt{d}$.}
  \begin{tabular}{lllll}
    {} & \algucbh{}(5e-2) & \algucbh{}(1e-2) & \algucbh{}(5e-3) \\
    \hline
    d = 1 (SE) &      4.72e-01 &     4.88e-01 &     5.08e-01 \\
    d = 2 (SE) &       9.59e+00 &     1.40e+01 &     1.61e+01 \\
    d = 3 (SE) &      1.77e+02 &     2.02e+02 &     2.02e+02 \\
  \end{tabular}
  \label{tab:time-se_alpha_l_small}
\end{table}




